\documentclass{article}

\usepackage{microtype}
\usepackage{graphicx}
\usepackage{subcaption}
\usepackage{contour}
\usepackage{booktabs} 
\usepackage[table,xcdraw]{xcolor}
\usepackage{multirow}

\usepackage{hyperref}

\usepackage[accepted]{icml2026}

\usepackage{amsmath}
\usepackage{amssymb}
\usepackage{mathtools}
\usepackage{amsthm}
\usepackage{bbm}
\usepackage{tikz}
\usepackage[normalem]{ulem}

\newcommand*\circled[1]{\tikz[baseline=(char.base)]{
            \node[shape=circle,draw,inner sep=1pt] (char) {#1};}}

\usepackage[capitalize,noabbrev]{cleveref}

\definecolor{firstbest}{HTML}{EA4335}   
\definecolor{secondbest}{HTML}{34A853}  
\definecolor{thirdbest}{HTML}{4285F4}   
\definecolor{catblue}{HTML}{1f77b4}   
\definecolor{catorange}{HTML}{ff7f0e} 
\definecolor{catgreen}{HTML}{2ca02c}  
\definecolor{catred}{HTML}{d62728}    
\newcommand{\catdot}[1]{%
  \tikz[baseline={(mark.base)}]{\node[circle, draw=black, line width=0.25pt, fill=#1, inner sep=2pt] (mark) {};}%
}

\theoremstyle{plain}
\newtheorem{theorem}{Theorem}[section]
\newtheorem{proposition}[theorem]{Proposition}
\newtheorem{lemma}[theorem]{Lemma}
\newtheorem{corollary}[theorem]{Corollary}
\theoremstyle{definition}
\newtheorem{definition}[theorem]{Definition}

\theoremstyle{remark}
\newtheorem{remark}[theorem]{Remark}

\usepackage[textsize=tiny]{todonotes}

\icmltitlerunning{Plain Transformers are Surprisingly Powerful Link Predictors}

\begin{document}

\twocolumn[
  \icmltitle{Plain Transformers are Surprisingly Powerful Link Predictors}



  \icmlsetsymbol{equal}{*}

  \begin{icmlauthorlist}
    \icmlauthor{Quang Truong}{msu}
    \icmlauthor{Yu Song}{msu}
    \icmlauthor{Donald Loveland}{snap}
    \icmlauthor{Mingxuan Ju}{snap}
    \icmlauthor{Tong Zhao}{snap}
    \icmlauthor{Neil Shah}{snap}
    \icmlauthor{Jiliang Tang}{msu}
  \end{icmlauthorlist}

  \icmlaffiliation{msu}{Department of Computer Science and Engineering, Michigan State University, East Lansing, MI, USA.}
  \icmlaffiliation{snap}{Snap Inc., Bellevue, WA, USA}

  \icmlcorrespondingauthor{Quang Truong}{truongc4@msu.edu}

  \icmlkeywords{Machine Learning, ICML}

  \vskip 0.3in
]



\printAffiliationsAndNotice{}  

\begin{abstract}

Link prediction is a core challenge in graph machine learning, demanding models that capture rich and complex topological dependencies. While Graph Neural Networks (GNNs) are the standard solution, state-of-the-art pipelines often rely on explicit structural heuristics or memory-intensive node embeddings---approaches that struggle to generalize or scale to massive graphs. Emerging Graph Transformers (GTs) offer a potential alternative but often incur significant overhead due to complex structural encodings, hindering their applications to large-scale link prediction. We challenge these sophisticated paradigms with PENCIL, an encoder-only plain Transformer that replaces hand-crafted priors with attention over sampled local subgraphs, retaining the scalability and hardware efficiency of standard Transformers. Through experimental and theoretical analysis, we show that PENCIL extracts richer structural signals than GNNs, implicitly generalizing a broad class of heuristics and subgraph-based expressivity. Empirically, PENCIL outperforms heuristic-informed GNNs and is far more parameter-efficient than ID-embedding--based alternatives, while remaining competitive across diverse benchmarks---even without node features. Our results challenge the prevailing reliance on complex engineering techniques, demonstrating that simple design choices are potentially sufficient to achieve the same capabilities. Our code is publicly available at \url{https://github.com/quang-truong/pencil}.

\end{abstract}

\section{Introduction} \label{sec:introduction}

Link prediction, the task of inferring missing edges between node pairs in graph-structured data, is a fundamental primitive with applications spanning recommendation systems \cite{huang_link_2005} and drug discovery \cite{Abbas2021ApplicationON}. Given its centrality in graph learning, Graph Neural Networks (GNNs) have become the de facto standard for this challenge. However, standard Message-Passing Neural Networks (MPNNs), the primary engine for most GNN-based encoders, are inherently limited by their node-centric aggregation, which often renders structurally distinct node pairs indistinguishable if the individual nodes possess symmetric local neighborhoods \cite{zhang_labeling_2022}. Therefore, state-of-the-art architectures typically integrate diverse structural signals, ranging from local heuristics \cite{wang2024neural, yun_neo-gnns_2021, chamberlain2023graph} and pairwise encodings \cite{li_distance_2020, zhang_link_2018} to global graph heuristics \cite{shomer_lpformer_2024} and learnable per-node embeddings \cite{ma_reconsidering_2025, dong_pure_2024}, through either implicit or explicit architectural modifications.

\begin{figure}[t]
    \centering
    \includegraphics[width=0.925\linewidth]{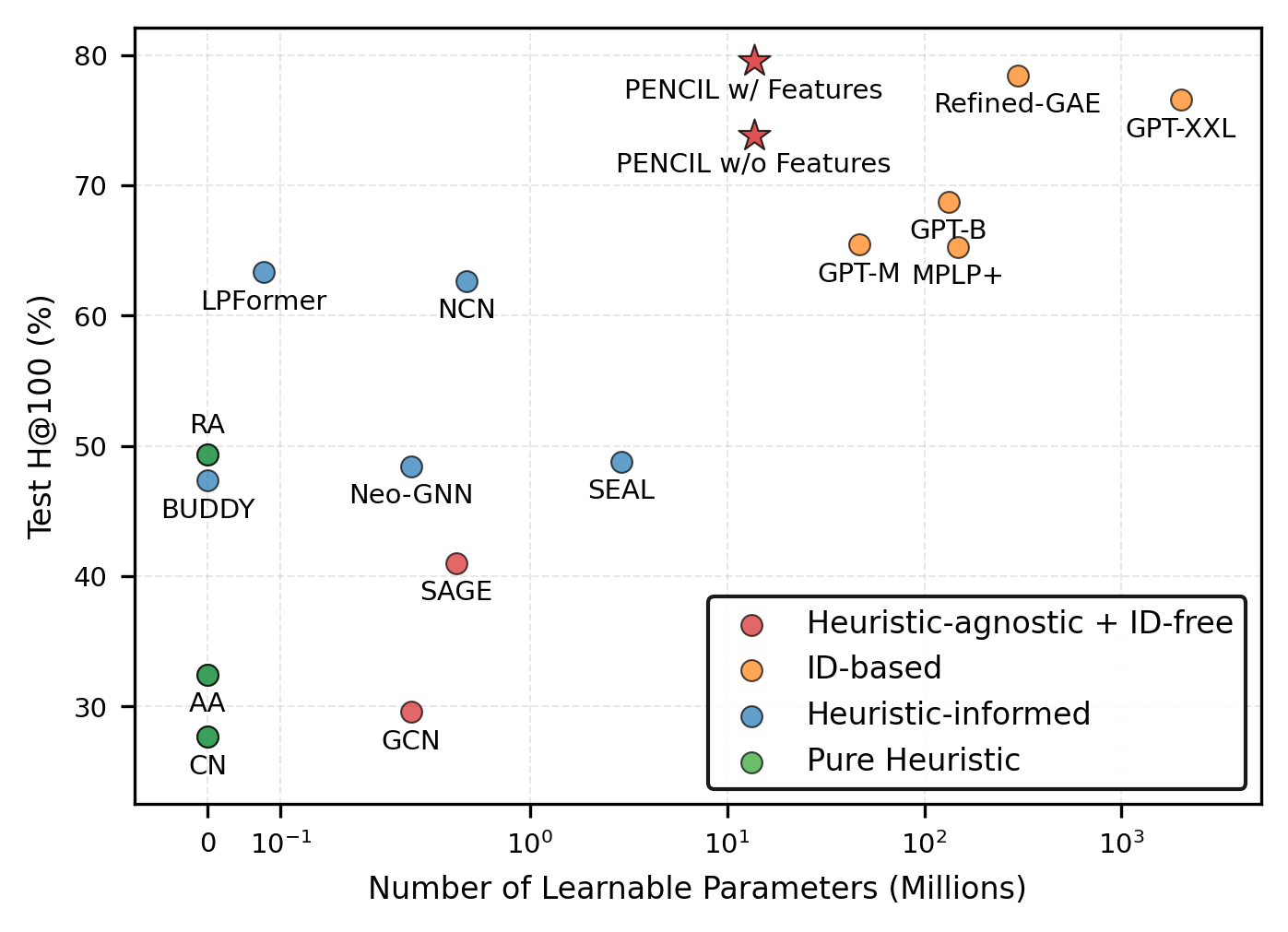}
    \caption{Parameter efficiency vs. performance on \texttt{ogbl-ppa}. PENCIL (\raisebox{0.15em}{\scalebox{0.7}{\contour{black}{\textcolor{firstbest}{$\bigstar$}}}}) achieves state-of-the-art performance without node IDs or handcrafted heuristics, using orders of magnitude fewer parameters than leading ID-based methods.}
    \label{fig:ogbl_ppa_comparison}
\end{figure}

While these augmented architectures address expressiveness, their design choices often compromise scalability and complicate deployment. First, reliance on learned per-node embeddings (ID-based) or precomputed global features often imposes a static dependency on the initial graph state, necessitating expensive retraining or recomputation to incorporate new nodes \cite{fey_rdl_2024,longa2023graph}. Second, at web scale, methods requiring whole-graph passes or the frequent refresh of globally materialized caches create prohibitive memory and update overhead; this necessitates predictors capable of operating on fixed-budget, mini-batch sampled neighborhoods \cite{zhao_gigl_2025, fey_rdl_2024}. Third, as noted by \cite{li_evaluating_2023}, reliance on simple structural cues like common neighbors often results in inflated performance metrics, allowing models to succeed through exploitation rather than by learning generalizable patterns. These limitations collectively motivate a shift toward link predictors designed for realistic deployment scenarios: ID-free, mini-batch efficiency, and structural robustness with minimal reliance on global preprocessing.

To address the limitations of GNNs, recent research has pivoted toward Graph Transformers (GTs). By replacing restricted neighborhood aggregation with all-to-all attention, GTs can be theoretically a universal function approximator under appropriate structural conditioning \cite{muller_attending_2024}. Yet, there is \textit{no consensus on how to bake graph structure into Transformers}. Specifically, researchers typically integrate positional (PE) and structural encodings (SE) derived from spectral signals \cite{rampasek2022GPS, kreuzer2021rethinking, dwivedi2021generalization}, random walks \cite{chen2025learning, kim_revisiting_2025}, graph heuristics \cite{ying_do_transformers_2021}, or GNNs \cite{jain2021representing, dwivedi2025relationalgraphtransformer}. However, while these encodings demonstrate proven success on graph-level tasks, these components are often infeasible for link prediction because they are processed offline or require entire graph structure, thus violating the constraints set above. The field also faces significant architectural heterogeneity, as there is no canonical way to adapt the \textit{plain Transformer} \cite{vaswani2017attention}---originally designed for sequential data---to irregular graph structures \cite{ma_plain_2025}. Current GT designs often necessitate specialized graph-specific attention kernels \cite{ma_graph_2023, shomer_lpformer_2024}, or sophisticated input processing pipelines that introduce substantial computational overhead \cite{chen2022structure, zhao_graphgpt_2025}. Such fundamental departures from standard architectures prevent these models from fully leveraging highly optimized hardware accelerator toolchains, such as efficient attention kernel \cite{dao2022flashattention}. Consequently, there is a clear vacancy for a Transformer-based link predictor that is expressive while adhering to the above deployment constraints.

In this work, we propose \textbf{\underline{P}lain \underline{ENC}oder for \underline{I}nferring \underline{L}inks} (PENCIL), a Transformer-based link predictor that utilizes a standard BERT-style encoder \cite{devlin-etal-2019-bert} for full compatibility with modern hardware. In particular, PENCIL evaluates each candidate link from a fixed-budget, sampled link-centric local neighborhood---without handcrafted heuristic features---enabling straightforward mini-batching while avoiding reliance on globally materialized caches. Notably, PENCIL \textbf{does not need costly offline computation of PEs/SEs}, which differentiates it from existing GTs. To our knowledge, this work provides the first demonstration that plain Transformers serve as highly effective link predictors under strict deployment constraints. Furthermore, we provide a formal theoretical analysis that characterizes the mechanisms underlying this performance and situates the model's expressive power relative to established link prediction paradigms. Main contributions are: 

\textbf{Transformer-based Link Predictor.~} We introduce a Transformer architecture that learns expressive representations from \emph{only} sampled local subgraphs. As highlighted in Figure~\ref{fig:ogbl_ppa_comparison}, it exceeds the performance of others while using $22\times$ to $146\times$ fewer learnable parameters than the next best competitors. Notably, PENCIL achieves a substantial advancement in training efficiency on large-scale datasets, requiring between $6.7\times$ and $40\times$ fewer epochs to converge than pure GNN architectures. Our results indicate that local, sampled contexts provide ample signal for high performance, highlighting a potential over-reliance in existing models on auxiliary components over inherent structural information.

\textbf{Theoretical Unification.~} We provide a theoretical analysis connecting PENCIL to established link prediction paradigms. We show that PENCIL inherently formulates many traditional structural heuristics by design, achieving the expressivity of subgraph-based predictors like SEAL \cite{zhang_link_2018} without explicit hard-coding distance-based labels.

\textbf{Data Insights.~} We find that, on several standard benchmarks, PENCIL can achieve strong performance using structural information alone. This suggests that node features can be weakly informative relative to local structure and may provide limited marginal gains on many datasets.

\section{Preliminaries} \label{sec:preliminaries}

In this section, we will introduce notations and concepts that will be used throughout the paper.

\textbf{Notations.~} Let $G=(\mathbf{A},\mathbf{X})$ denote an (undirected) graph with adjacency matrix $\mathbf{A}\in\{0,1\}^{M\times M}$ and node-feature matrix $\mathbf{X}\in\mathbb{R}^{M\times f}$, where \(M\) is the number of nodes and \(f\) is the feature dimension. Let degree matrix be $\mathbf{D}$. The node set is $\mathcal{V}=\{1,\ldots,M\}$ and the edge set is $\mathcal{E}=\{(i,j)\in\mathcal{V}\times\mathcal{V}:\mathbf{A}_{ij}=1\}$. For $v\in\mathcal{V}$, we define the neighborhood as $\mathcal{N}(v)=\{u\in\mathcal{V}:\mathbf{A}_{uv}=1\}$. For any matrix square $\mathbf{M}$, its $i$-th row is denoted by $\mathbf{m}_i=\mathbf{M}_i$. Let \(\mathbf{P}_\pi\) denote permutation matrix of node relabeling \(\pi\).

\textbf{Link Prediction.~}
Given an observed graph $G=(\mathbf{A},\mathbf{X})$ with edge set $\mathcal{E}$, the link prediction task assigns a score to a candidate node pair $(u,v)\in\mathcal{V}\times\mathcal{V}$ indicating the likelihood that an edge exists between $u$ and $v$. Under a temporal split, $\mathcal{E}$ contains edges observed up to a cutoff time, while evaluation targets future edges; thus, the edges to be predicted are not necessarily a subset of $\mathcal{E}$.

\textbf{Pairwise Heuristics.~}
A pairwise heuristic is a scoring function $h$ that, given the observed graph $G$, maps a candidate node pair $(u,v)\in\mathcal{V}\times\mathcal{V}$ to a real-valued score. Candidate pairs are then ranked by this score to estimate the likelihood that an edge exists between $u$ and $v$. Common pairwise heuristics are summarized in Appendix~\ref{app:pairwise-heuristic-estimation}.

\begin{figure}[t]
    \centering
    \includegraphics[width=0.95\linewidth]{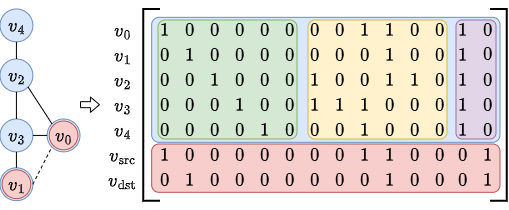}
    \caption{Visualization of the input encoding scheme for a sampled subgraph revolving around a query link (dashed), with the number of nodes $N=5$ and the sampling budget $N_{\max}=6$. Blue block corresponds to the sampled nodes (the context set). Green block contains one-hot identifiers of nodes, yellow block contains nodes' adjacency row, purple block contains role flags, and red block contains task tokens. Context-node ordering is permutable; only $v_{\mathrm{src}}$ and $v_{\mathrm{dst}}$ are fixed to \(v_0\) and \(v_1\).
    }
    \label{fig:adjacency_row}
\end{figure}

\begin{figure}[t]
    \centering
    \includegraphics[width=0.85\linewidth]{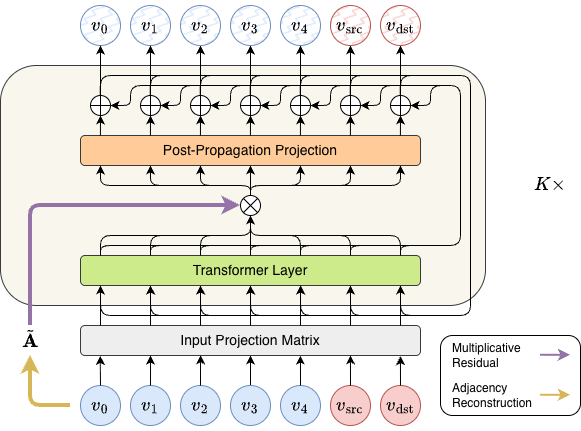}
    \caption{PENCIL architecture. Node features can be optionally used as discussed in Section~\ref{sec:architectural-details}, but we omit them from the figure for clarity.}
    \label{fig:PENCIL}
\end{figure}

\section{Plain Transformers As Link Predictors} \label{sec:plain-transformers-as-link-predictors}
In this section, we describe graph encoding scheme, the architecture of PENCIL, and its use for estimating pairwise heuristics relative to classic GNNs.

\subsection{Graph Encoding Scheme}\label{sec:graph-encoding-scheme}
We encode each sampled subgraph around a candidate pair $(v_{\text{src}}, v_{\text{dst}})$ using an encoding scheme similar to node-adjacency tokenization \cite{yehudai_depth-width_2025}. As illustrated in Figure~\ref{fig:adjacency_row}, we assign indices to the $N$ nodes in the extracted subgraph (with $N \le N_{\max}$ where $N_{\max}$ is the sampling budget) by fixing $v_{\text{src}}$ to index $0$ and $v_{\text{dst}}$ to index $1$, while assigning the remaining nodes uniformly at random to indices $2,\ldots,N-1$. For the first $N$ \emph{context} nodes (blue block), we concatenate three parts: (i) a one-hot identifier of the node in the chosen subgraph order, padded to length $N_{\max}$ (green block); (ii) the node's adjacency indicator row within the extracted subgraph, also padded to length $N_{\max}$ (yellow block); and (iii) a two-bit role flag indicating whether the node is a context node or a task node (purple block), set to $[1,0]$ for context nodes. We then append two \emph{task} nodes corresponding to the endpoints $v_{\text{src}}$ and $v_{\text{dst}}$ (bottom two red rows): each task node copies the same one-hot and adjacency parts from its corresponding context node representation, but flips the role flag to $[0,1]$. Since `node' and `token' are referring to the same entity, we use them interchangeably from now on.

\subsection{Architecture Design} \label{sec:architectural-details}

We begin by outlining the subgraph batching, which will clarify the tensor dimensions used later. Each mini-batch contains $B$ sampled subgraphs; letting $N_b$ be the number of context nodes in subgraph $b$, we set $N_B=\max_{b\in[B]} N_b$ and pad all samples to this length. The model processes a length-$(N_B+2)$ token sequence consisting of $N_B$ context tokens (sampled subgraph nodes, with padding as needed) and two task tokens (for $v_{\text{src}}$ and $v_{\text{dst}}$). By construction, $N_B \le N_{\max}$, where $N_{\max}$ is the global sampling budget. It is worth noting that there are two levels of padding: padding to \(N_{\max}\) is feature-level padding, used for the one-hot identifier and adjacency fields, whereas padding to \(N_B+2\) is batch-level sequence padding, accounting for the two appended task tokens. While one can pad every sequence to \(N_{\max}+2\), it is more efficient to pad to \(N_B+2\), avoiding unnecessary computation when sampled subgraphs are smaller than \(N_{\max}\).

Figure~\ref{fig:PENCIL} illustrates PENCIL, a model composed of $K$ stacked blocks of identical form that couple bidirectional self-attention with one-hop propagation. We omit the default sequential PE, so the Transformer encoder is permutation equivariant with respect to reordering the input tokens.

\textbf{Input Projection.~}
The tokenized input \(\tilde{\mathbf{X}}\in\mathbb{R}^{B\times (N_B+2)\times (2N_{\max}+2)}\) (see Figure~\ref{fig:adjacency_row}) and optional node features $\mathbf{X} \in \mathbb{R}^{B\times (N_B+2) \times f}$ are projected into a shared hidden dimension $d$ to form $\mathbf{H}^{(0)}$:
\begin{equation}
\mathbf{H}^{(0)}=\tilde{\mathbf{X}}\mathbf{W}_0 + \mathbf{G}\!\left(\mathbf{X}\right) \in \mathbb{R}^{B\times (N_B+2)\times d},
\label{eq:input-proj}
\end{equation}
where $\mathbf{W}_0 \in \mathbb{R}^{(2N_{\max}+2) \times d}$ is an orthogonally initialized projection matrix and $\mathbf{G}(\cdot)$ is a learnable feature encoder. Although random initialization often yields approximately orthogonal vectors in high dimensions, it does not guarantee an orthonormal or norm-preserving projection. We therefore use orthogonal initialization~\cite{saxe2014exact} with gain 1 to obtain a better-conditioned initial projection.

\textbf{Adjacency Reconstruction.~}
A key feature of PENCIL is that the subgraph adjacency matrix \(\tilde{\mathbf{A}}\) is not provided as a separate input but is recovered directly from the token encoding \(\tilde{\mathbf{X}}\). Since each token includes its adjacency-indicator row (see Figure~\ref{fig:adjacency_row}), \(\tilde{\mathbf{A}}\) is constructed by slicing the corresponding columns from \(\tilde{\mathbf{X}}\). 
While normalized variants such as $\mathbf{D}^{-1}\tilde{\mathbf{A}}$ or $\mathbf{D}^{-1/2}\tilde{\mathbf{A}}\mathbf{D}^{-1/2}$ are applicable, we retain the raw adjacency $\tilde{\mathbf{A}}$ for notational simplicity. Details regarding the adjacency reconstruction are discussed in Appendix~\ref{app:adjacency-recovery}.

\textbf{Multiplicative Residual.~}
For each layer $k$, PENCIL applies a bidirectional Transformer block followed by a residual term:

\begin{equation}
    \mathbf{Z}^{(k)} = \mathbf{T}_{k}\!\left(\mathbf{H}^{(k-1)}\right)
    \in \mathbb{R}^{(N_B+2)\times d}, \label{eq:Z-k}
\end{equation}
\begin{equation}
    \mathbf{H}^{(k)} = \mathbf{Z}^{(k)} + \mathbf{P}_{k}\!\left(\tilde{\mathbf{A}}\mathbf{Z}^{(k)}\right)
    \in \mathbb{R}^{(N_B+2)\times d},
    \label{eq:H-k}
\end{equation}

where \(\mathbf{T}_k\) is a Transformer layer and \(\mathbf{P}_k\) is a learnable row-wise post-propagation projection applied to \(\tilde{\mathbf{A}}\mathbf{Z}^{(k)}\). We refer to the second term as a \textbf{multiplicative residual} because it adds an explicit matrix-multiplication branch \(\mathbf{P}_k(\tilde{\mathbf{A}}\mathbf{Z}^{(k)})\) on top of the attention output, where \(\tilde{\mathbf{A}}\) is reconstructed from the input encoding \(\tilde{\mathbf{X}}\). Therefore, if self-attention is replaced with the identity map \(\mathbf{Z}^{(k)}=\mathbf{H}^{(k-1)}\), the update reduces to a residual message-passing layer over \(\tilde{\mathbf{A}}\), with node features induced by the graph encoding scheme in Section~\ref{sec:graph-encoding-scheme}. As the residual is computed outside the self-attention module, resulting in full compatibility with efficient attention techniques. Finally, self-attention is computed only among non-padded tokens from the same sampled subgraph. 

The link logit is obtained by concatenating the final representations of $v_{\text{src}}$ and $v_{\text{dst}}$, followed by a linear projection to a scalar. The model is trained with the binary cross-entropy (BCE) loss.

\subsection{Estimating Pairwise Heuristics with PENCIL} \label{sec:pairwise-heuristic-estimation}

Unlike link prediction, which depends on many factors beyond topology, pairwise heuristics are determined by graph structure alone. Furthermore, prior work has established theoretical links between these heuristics and link prediction performance, explaining why they can already serve as strong baselines \cite{sarkar2010theoretical}. Pairwise heuristics therefore provide a clean benchmark for assessing how well Transformers can recover fundamental structural signals and where their limitations may lie. In this section, we evaluate PENCIL on a \emph{pairwise heuristic estimation} task and compare against standard GNN baselines.

For each candidate pair \((u,v)\), we compute a heuristic score on the \emph{full} graph and use it as a scalar regression target. We include both \emph{local} heuristics---Common Neighbors (CN), Adamic--Adar (AA), and Resource Allocation (RA)---and \emph{global} heuristics that aggregate multi-hop or graph-wide structure, namely Katz index \cite{katz_new_1953}, shortest-path distance (SPD), and a PageRank-based pair score \(\mathrm{PR}(u,v)=p_u p_v\) \cite{brin_anatomy_1998}. In contrast, during training and evaluation all models are restricted to an induced subgraph \emph{centered around} \((u,v)\) produced by the sampling procedure, and are not given access to the full graph. The model must therefore approximate full-graph statistics using only restricted local subgraphs. Our aim is to quantify, under the sampling constraints, how accurately different architectures can recover these heuristic signals. Figure~\ref{fig:cora_heuristic_rmse} reports RMSE on \texttt{cora}; additional setup details and \texttt{citeseer} results are provided in Appendix~\ref{app:pairwise-heuristic-estimation}.

\begin{figure}[t]
    \centering
    \includegraphics[width=\linewidth]{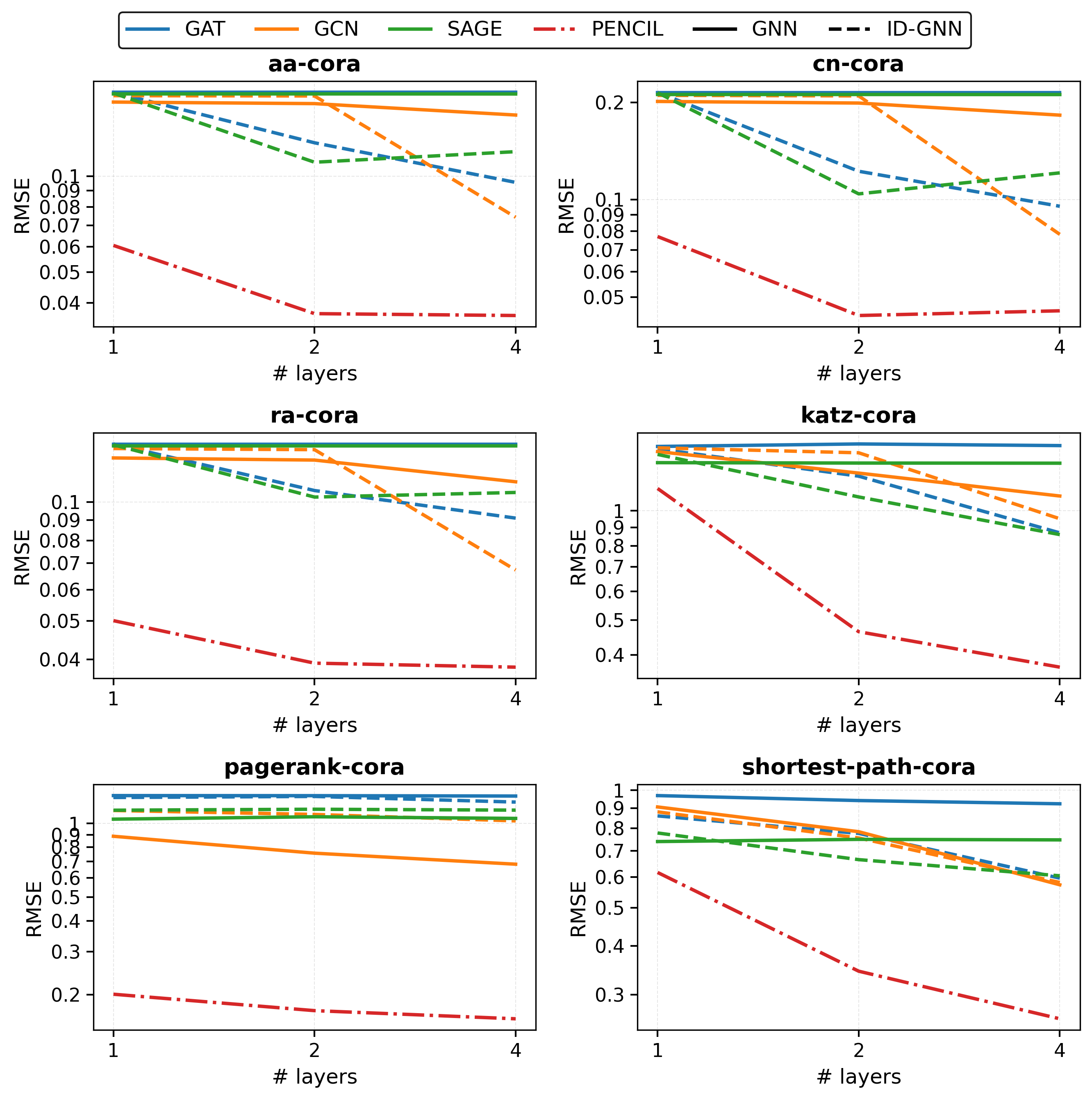}
    \caption{RMSE of PENCIL and other GNNs for estimating pairwise heuristics on the \texttt{cora} dataset.}
    \label{fig:cora_heuristic_rmse}
\end{figure}

Based on Figure~\ref{fig:cora_heuristic_rmse}, we observe three distinct trends. First, for local heuristics (AA, CN, RA), standard GNNs show negligible improvement with depth. In contrast, ID-GNNs significantly reduce error within the first two layers before performance saturates. Second, for long-range heuristics (Katz index, SPD), increased depth yields clearer benefits, most notably for PENCIL, where RMSE steadily decreases as the model deepens. Third, PageRank proves to be the most challenging target for MPNN baselines; neither ID augmentation nor added depth bridges the performance gap, whereas PENCIL achieves substantially lower RMSE across all depths. Notably, the same observation regarding GNNs' subpar performance is reported in \cite{liang_can_2025}, where they suggest that GNNs need learnable node embeddings to learn pairwise heuristics.

Overall, PENCIL surpasses all GNN and ID-GNN models even with a single layer, achieves low error on local metrics with just two layers, and effectively leverages depth to capture complex properties like Katz index, SPD, and PageRank. Similar trends are observed on \texttt{citeseer} (see Appendix~\ref{app:pairwise-heuristic-estimation}). Collectively, these results indicate that under identical sampling constraints, PENCIL extracts structural information more effectively than MPNNs, validating its potential as a robust link predictor. Furthermore, many state-of-the-art methods incorporate these heuristics as explicit input signals \cite{wang2024neural,yun_neo-gnns_2021,dong_pure_2024}; in contrast, our results show that a Transformer can estimate them directly from the input graph, without hard-coding them as features.

\section{Theoretical Analysis} \label{sec:theoretical-analysis}
In this section, we would like to provide theoretical analysis on PENCIL, particularly how it can work on link prediction tasks, its surprising capability on pairwise heuristic estimation, and its expressive power with respect to other link prediction models.

\subsection{Distributional Permutation Invariance}
In graph learning, it is standard to require symmetry with respect to node relabelings. Concretely, intermediate representations should be \emph{permutation equivariant}, whereas task-level predictions should be \emph{permutation invariant}—meaning the output remains unchanged under arbitrary node permutations \cite{zaheer_deep_2017,maron2018invariant,bodnar_weisfeiler_2021,truong2024weisfeiler}. 

PENCIL is not \emph{deterministically} invariant to node relabeling. This is because adjacency-row tokenization relies on a randomized index assignment: it fixes the queried endpoints to canonical positions and assigns random indices to the remaining nodes. While this breaks deterministic invariance for a fixed sample, it does not violate the underlying symmetry required for graph tasks. We show that the induced predictor is \emph{permutation invariant in distribution} over the randomness of this assignment.

\begin{theorem}\label{thm:dist_perm_invariance}
Let $f$ be any deterministic measurable function. Define a randomized predictor
$S(\mathbf{A}; u, v) := f(\mathbf{P}_\rho \mathbf{A} \mathbf{P}_\rho^\top)$,
where $\rho$ is drawn uniformly from the set of permutations satisfying $\rho(u)=0$ and $\rho(v)=1$.
Then for any node relabeling $\pi$, the output distribution is unchanged:
\[
    S(\mathbf{A}; u, v) \;\overset{d}{=}\; S(\mathbf{A}'; u', v'),
\]
where $\mathbf{A}' = \mathbf{P}_\pi \mathbf{A} \mathbf{P}_\pi^\top$, $u' = \pi(u)$, and $v' = \pi(v)$.
\end{theorem}

Detailed proof is included in Appendix~\ref{app:proof_dist_perm_invariance}. Theorem~\ref{thm:dist_perm_invariance} formalizes that, although our encoding is not deterministically invariant for a fixed canonicalization, the resulting randomized predictor remains a valid graph function: relabeling the graph and the query pair leaves the predictor unchanged \emph{in distribution}.

Averaging over multiple independent canonicalizations can reduce the variance induced by random index assignments and better approximate the invariant predictor, but it requires additional forward passes at inference time. In practice, a single labeled subgraph per forward pass already achieves strong performance, as shown in Section~\ref{sec:experimental-results}. Appendix~\ref{app:multiple_labeled_subgraphs} further shows that using multiple labeled subgraphs with DeepSets-style pooling~\cite{zaheer_deep_2017} does not yield clear gains, supporting the efficiency of our design.


\subsection{What Makes PENCIL a Strong Link Predictor?}

In this section, we draw connections between PENCIL and prior works on link prediction, namely NBFNet \cite{zhu_neural_bellman-ford_2021} and MPLP \cite{dong_pure_2024}, to provide theoretical evidence for its surprising capability on pairwise heuristic estimation demonstrated in Section~\ref{sec:pairwise-heuristic-estimation}.

NBFNet \cite{zhu_neural_bellman-ford_2021} is a source-conditioned message-passing framework \cite{you2021identity} for link prediction that propagates information from a query source node and scores links by reading out the representations at candidate destination nodes. It can be viewed as a learnable relaxation of the generalized Bellman--Ford algorithm that recovers many path-based heuristics; details are provided in Appendix~\ref{app:nbfnet}. Next, we show that PENCIL can be degenerated to an NBFNet-style propagation model, subsequently leading to the following corollary.

\begin{proposition}
    \label{prop:PENCIL-degenerates-nbfnet}
    There exists a parameter setting of PENCIL under which its layerwise update reduces to a source-conditioned MPNN, with readout at the canonical destination token \(v_1\).
\end{proposition}

\begin{corollary} \label{cor:global-heuristics}
    Under suitable parameter settings and operator choices, PENCIL can realize a broad class of classical path-based link prediction scores and graph algorithms, including Katz index, Personalized PageRank, SPD, widest path, and most reliable path.
\end{corollary}

Proofs are provided in Appendix~\ref{app:proof_PENCIL-degenerates-nbfnet} and Appendix~\ref{app:proof_global-heuristics}. Next, we show that PENCIL can also estimate local pairwise heuristics. First, we have the following remark, which is an informal statement of Theorem~\ref{thm:common-neighbor-estimation} and Theorem~\ref{thm:common-neighbor-estimation-general} in Appendix~\ref{app:formal-theorems}, proven in \cite{dong_pure_2024}.

\begin{remark} \label{remark:dot-products}
    If the initial node features are random, zero-mean, and (in expectation) unit-norm, then sum-aggregation message passing turns dot products into counts of shared connectivity patterns.
\end{remark}

The above remark suggests that sum-aggregation MPNNs can estimate local heuristics, depending on the initial node representations---namely, zero mean and unit norm in expectation. Importantly, these conditions are sufficient but not restrictive in practice: they are already satisfied (or closely approximated) by common ways of initializing learnable node embeddings, such as zero-mean random initialization with appropriate scaling, and in particular orthogonal initialization. Motivated by this perspective, \cite{dong_pure_2024} studies quasi-orthogonal node signatures, while \cite{ma_reconsidering_2025} initializes learnable node embeddings with orthogonal vectors. We now show that PENCIL can be configured to satisfy the same assumptions, and therefore inherits the same unbiased estimation property.

\begin{proposition} \label{prop:local-heuristics}
    There exists a parameter setting of PENCIL under which it can estimate heuristics stated in Theorem~\ref{thm:common-neighbor-estimation} and Theorem~\ref{thm:common-neighbor-estimation-general}.
\end{proposition}

Proof is provided in Appendix~\ref{app:proof_local-heuristics}. A key distinction from MPLP \cite{dong_pure_2024} and Refined-GAE \cite{ma_reconsidering_2025} is where these vectors live. Those methods conceptually maintain a (quasi-)orthogonal vector per node, whereas in PENCIL the vectors are induced by the input projection over the \emph{sampled} token set. Consequently, the number of vectors that must be simultaneously well-separated is bounded by the sampling budget (i.e., \(N_{\max} \ll |\mathcal V|\)). Since the accuracy of the above unbiased estimators depends on orthogonality of vectors, we quantify this property via mutual coherence
\cite{donoho_uncertainty_2001} and the Welch bound \cite{welch_lower_1974}.

\begin{definition}[Mutual coherence]
    Let $\mathcal{F} = \{\mathbf{r}_1, \dots, \mathbf{r}_N\}$ be a set of $N$ unit-norm vectors in $\mathbb{R}^d$. The \textit{mutual coherence} of $\mathcal{F}$, denoted $\mu(\mathcal{F})$, is defined as the maximum absolute inner product between any distinct pair of vectors in the set:
    \begin{equation}
        \mu(\mathcal{F}) \triangleq \max_{1 \leq i \neq j \leq N} |\langle \mathbf{r}_i, \mathbf{r}_j \rangle|.
    \end{equation}
\end{definition}

\begin{definition}[The Welch bound]
For any set size $N$ and dimension $d$ (where $N > d$), the \textit{Welch bound} $W(N, d)$ defines the theoretical lower bound on the mutual coherence achievable by any set of vectors:
\begin{equation}
    \mu(\mathcal{F}) \geq W(N, d) \triangleq \sqrt{\frac{N-d}{d(N-1)}}.
\end{equation}
\end{definition}

Mutual coherence captures the maximum correlation within a set of unit vectors, quantifying the set's deviation from orthogonality. The Welch bound lower-bounds this quantity as a function of the set size \(N\) and dimension \(d\), yielding a
fundamental limit on how well-separated \(N\) vectors in \(\mathbb{R}^d\) can be. We have the following proposition.

\begin{proposition} \label{prop:monotonicity-of-welch-bound}
    For a fixed dimension $d \ge 2$, the Welch bound $W(N,d)$ is strictly increasing as a function of the set size $N$ for all $N > d$.
\end{proposition}

Proof is provided in Appendix~\ref{app:proof_monotonicity-of-welch-bound-proposition}. This monotonicity reveals an inherent benefit of operating with fewer vectors: for fixed \(d\), the smallest achievable mutual coherence increases with \(N\), meaning that it becomes increasingly difficult to keep all pairwise correlations small as the set grows. Since PENCIL only needs to keep \(N_{\max}\) vectors---rather than maintaining \(|\mathcal V|\) well-separated node vectors globally---it operates under a strictly weaker coherence constraint. In particular, when \(N_{\max} \le d\), perfectly orthogonal vectors are feasible, which can arise in very large but sparse graphs. In contrast, maintaining strictly orthogonal vectors at the scale of \(|\mathcal V|\) is infeasible because \(|\mathcal V| \gg d\).

\subsection{Expressive Power of Plain Transformers on Link Prediction}

Beyond heuristic estimation, PENCIL's expressivity can be probed on highly symmetric graphs by asking whether it distinguishes candidate links whose endpoints are automorphic (structurally interchangeable). Standard node-embedding pipelines often collapse automorphic nodes---and thus such links---to identical representations \cite{chamberlain2023graph}, while prior work breaks this symmetry via node labeling \cite{zhang_link_2018, zhang_labeling_2022}.

PENCIL is imbued with a native symmetry-breaking mechanism: the input projection assigns random vectors to tokens, inducing a non-deterministic node labeling. Since each prediction is computed on a sampled subgraph of size at most \(N_{\max}\), only finitely many labeled versions of that subgraph arise. This places PENCIL naturally in the family of subgraph-based GNNs \cite{murphy_relational_2019, chen_can_2020,zhou_relational_2023, zhao2022from, zhao2022a}. Moreover, equipping PENCIL with local relational pooling (LRP) \cite{chen_can_2020}---i.e., pooling representations across these labeled subgraphs---allows us to evaluate its expressivity within the \(k_\phi\text{-}k_\rho\text{-}m\) framework \cite{lachi2025bridgingtheorypracticelink} and to compare it formally against other paradigms. We leave details about the framework, proof of the following theorem, and further discussion in Appendix~\ref{app:expressive-power-of-PENCIL}.

\begin{theorem} \label{thm:expressivity-of-PENCIL}
    Under the \(k_\phi\text{-}k_\rho\text{-}m\) framework, for suitable parameter settings, PENCIL with LRP is not less expressive than SEAL under the same sampling constraint.
\end{theorem}

Importantly, \cite{zhou_relational_2023} prove that labeling more nodes in the subgraph yields strictly stronger expressiveness; thus, increasing the sampling budget \(N_{\max}\) directly strengthens PENCIL's expressive power.

\section{Experimental Results} \label{sec:experimental-results}

\begin{table*}[t]
    \centering
    \caption{Results on the original benchmark datasets. 
    N/A means results are not available. Colored results indicate the \textcolor{firstbest}{\(1^{\text{st}}\)}, \textcolor{secondbest}{\(2^{\text{nd}}\)}, and \textcolor{thirdbest}{\(3^{\text{rd}}\)}-best performance in the corresponding metric. Category markers: 
    \raisebox{0.75ex}{\catdot{catblue}} Heuristic-informed, \raisebox{0.75ex}{\catdot{catorange}} ID-based, and \, \raisebox{0.75ex}{\catdot{catred}} Heuristic-agnostic + ID-free.
    }

    \label{tab:original-setting}
    \small
    \resizebox{\textwidth}{!}{%
    \begin{tabular}{@{}clcccccc@{}}
    \toprule
    {\multirow{2}{*}{Type}} & \multirow{2}{*}{Model} & \texttt{cora}                                & \texttt{citeseer}                            & \texttt{pubmed}                              & \texttt{ogbl-collab}                         & \texttt{ogbl-ppa}                            & \texttt{ogbl-citation2}                      \\ \cmidrule(l){3-8} 
             &                               & MRR                                 & MRR                                 & MRR                                 & H@50                                & H@100                               & MRR                                 \\ \midrule
    \multirow{6}{*}{\raisebox{0.75ex}{\catdot{catblue}}} 
             & SEAL                       & 26.69 ± 5.89                        & 39.36 ± 4.99                        & 38.06 ± 5.18                        & 63.37 ± 0.69                        & 48.80 ± 5.61                        & 86.93 ± 0.43                        \\
             & Neo-GNN                    & 22.65 ± 2.60                        & 53.97 ± 5.88                        & 31.45 ± 3.17                        & 66.13 ± 0.61                        & 48.45 ± 1.01                        & 83.54 ± 0.32                        \\
             & BUDDY                      & 26.40 ± 4.40                        & {\color{thirdbest} 59.48 ± 8.96}    & 23.98 ± 5.11                        & 64.59 ± 0.46                        & 47.33 ± 1.96                        & 87.86 ± 0.18                        \\
             & NCN                        & 32.93 ± 3.80                        & 54.97 ± 6.03                        & 35.65 ± 4.60                        & 63.86 ± 0.51                        & 62.63 ± 1.15                        & 89.27 ± 0.05                        \\
             & NCNC                       & 29.01 ± 3.83                        & {\color{secondbest} 64.03 ± 3.67}   & 25.70 ± 4.48                        & 65.97 ± 1.03                        & 62.61 ± 0.76                        & {\color{secondbest} 89.82 ± 0.43}   \\
             & LPFormer                   & {\color{secondbest} 39.42 ± 5.78}   & {\color{firstbest} 65.42 ± 4.65}    & {\color{secondbest} 40.17 ± 1.92}   & {\color{firstbest} 68.14 ± 0.51}    & 63.32 ± 0.63                        & {\color{thirdbest} 89.81 ± 0.13}    \\ \midrule
    \multirow{2}{*}{\raisebox{0.75ex}{\catdot{catorange}}} 
             & MPLP+                      & N/A                                 & N/A                                 & N/A                                 & {\color{secondbest} 66.99 ± 0.40}   & 65.24 ± 1.50                        & {\color{firstbest} 90.72 ± 0.12}    \\
             & Refined-GAE                & N/A                                 & N/A                                 & N/A                                 & 66.11 ± 0.35                        & {\color{secondbest} 78.41 ± 0.83}   & 88.74 ± 0.06                        \\ \midrule
    \multirow{5}{*}{\raisebox{0.75ex}{\catdot{catred}}}
             & GCN                        & 32.50 ± 6.87                        & 50.01 ± 6.04                        & 19.94 ± 4.24                        & 54.96 ± 3.18                        & 29.57 ± 2.90                        & 84.85 ± 0.07                        \\
             & SAGE                       & {\color{thirdbest} 37.83 ± 7.75}    & 47.84 ± 6.39                        & 22.74 ± 5.47                        & 59.44 ± 1.37                        & 41.02 ± 1.94                        & 83.06 ± 0.09                        \\
             & NBFNet                     & 37.69 ± 3.97                        & 38.17 ± 3.06                        & {\color{firstbest} 44.73 ± 2.12}    & N/A                                 & N/A                                 & N/A                                 \\
             & PENCIL w/o Features (Ours) & {\color{firstbest} 42.23 ± 1.98}    & 47.51 ± 3.09                        & 38.28 ± 2.59                        & {\color{thirdbest} 66.88 ± 0.34}    & {\color{thirdbest} 73.85 ± 0.40}    & 86.74 ± 0.26                        \\
             & PENCIL (Ours)             & 32.12 ± 3.04                        & 43.74 ± 5.47                        & {\color{thirdbest} 38.34 ± 5.14}    & 66.56 ± 0.19                        & {\color{firstbest} 79.54 ± 0.07}    & 86.86 ± 0.20                        \\ \bottomrule
    \end{tabular}%
    }
\end{table*}

This section highlights findings: (1) a plain Transformer's surprisingly strong performance on real-world datasets and (2) the impact of model depth under fixed sampling constraints. Further investigations on multiplicative residual, input projection matrix initialization, and computational complexity are included in Appendix~\ref{app:mr-ablation-study}, Appendix~\ref{app:effects-of-input-projection-matrix-initialization}, and Appendix~\ref{app:computational-complexity}, respectively.

Regarding experiment settings, we repeat each experiment with five random seeds on the Planetoid benchmarks (\texttt{cora}, \texttt{citeseer}, and \texttt{pubmed}) \cite{yang_revisiting_2016} and three random seeds on the OGB link prediction datasets (\texttt{ogbl-citation2}, \texttt{ogbl-ppa}, \texttt{ogbl-ddi}, and \texttt{ogbl-collab}) \cite{hu2020ogb}. Except to the main results, all experiments are conducted without node features. Additionally, \texttt{ogbl-ddi} does not have node features. Implementation information is detailed in Appendix~\ref{app:experiments}.

\begin{table*}[t]
    \caption{Results under HeaRT. N/A means results are not available. Colored results indicate the \textcolor{firstbest}{\(1^{\text{st}}\)}, \textcolor{secondbest}{\(2^{\text{nd}}\)}, and \textcolor{thirdbest}{\(3^{\text{rd}}\)}-best performance in MRR. Category markers: 
    \raisebox{0.75ex}{\catdot{catblue}} Heuristic-informed, \raisebox{0.75ex}{\catdot{catorange}} ID-based, and \, \raisebox{0.75ex}{\catdot{catred}} Heuristic-agnostic + ID-free.}
    \label{tab:heart}
    \resizebox{\textwidth}{!}{%
    \begin{tabular}{@{}clccccccc@{}}
    \toprule
    Type & Model              & \texttt{cora}                                & \texttt{citeseer}                            & \texttt{pubmed}                             & \texttt{ogbl-collab}                        & \texttt{ogbl-ppa}                            & \texttt{ogbl-ddi}                            & \texttt{ogbl-citation2}                      \\ \midrule
    \multirow{6}{*}{\raisebox{0.75ex}{\catdot{catblue}}} 
             & SEAL                & 10.67 ± 3.46                        & 13.16 ± 1.66                        & 5.88 ± 0.53                        & {6.43 ± 0.32} & 29.71 ± 0.71                        & 9.99 ± 0.90                         & 20.60 ± 1.28                        \\
             & Neo-GNN             & 13.95 ± 0.39                        & 17.34 ± 0.84                        & 7.74 ± 0.30                        & 5.23 ± 0.90                        & 21.68 ± 1.14                        & 10.86 ± 2.16                        & 16.12 ± 0.25                        \\
             & BUDDY               & 13.71 ± 0.59                        & 22.84 ± 0.36                        & 7.56 ± 0.18                        & 5.67 ± 0.36                        & 27.70 ± 0.33                        & 12.43 ± 0.50                        & 19.17 ± 0.20                        \\
             & NCN                 & 14.66 ± 0.95                        & {\color[HTML]{EA4335} 28.65 ± 1.21} & 5.84 ± 0.22                        & 5.09 ± 0.38                        & 35.06 ± 0.26                        & 12.86 ± 0.78                        & 23.35 ± 0.28                        \\
             & NCNC                & {\color[HTML]{4285F4} 14.98 ± 1.00} & {\color[HTML]{4285F4} 24.10 ± 0.65} & {8.58 ± 0.59} & 4.73 ± 0.86                        & 33.52 ± 0.26                        & N/A                                 & 19.61 ± 0.54                        \\
             & LPFormer            & {\color[HTML]{EA4335} 16.80 ± 0.52} & {\color[HTML]{34A853} 26.34 ± 0.67} & {\color[HTML]{EA4335} 9.99 ± 0.52} & {\color{thirdbest} 7.62 ± 0.26} & 40.25 ± 0.24 & {\color[HTML]{4285F4} 13.20 ± 0.54} & {\color[HTML]{EA4335} 24.70 ± 0.55} \\ \midrule
    \multirow{1}{*}{\raisebox{0.75ex}{\catdot{catorange}}} 
             & MPLP+               & N/A                                 & N/A                                 & N/A                                & {6.79}                               & \color{thirdbest}{41.40}                               & N/A                                 & 23.11                               \\ \midrule
    \multirow{5}{*}{\raisebox{0.75ex}{\catdot{catred}}} 
             & GCN                 & {\color[HTML]{34A853} 16.61 ± 0.30} & 21.09 ± 0.88                        & 7.13 ± 0.27                        & 6.09 ± 0.38                        & 26.94 ± 0.48                        & {\color[HTML]{34A853} 13.46 ± 0.34} & 19.98 ± 0.35                        \\
             & SAGE                & 14.74 ± 0.69                        & 21.09 ± 1.15                        & {\color[HTML]{34A853} 9.40 ± 0.70} & 5.53 ± 0.50                        & 27.27 ± 0.30                        & 12.60 ± 0.72                        & 22.05 ± 0.12                        \\
             & NBFNet              & 13.56 ± 0.58                        & 14.29 ± 0.80                        & N/A                                & N/A                                & N/A                                 & N/A                                 & N/A                                 \\
             & PENCIL w/o Features (Ours) & 14.63 ± 0.52                        & 16.50 ± 0.31                        & 7.05 ± 0.17                        & {\color{secondbest} 8.13 ± 0.09}                        & {\color[HTML]{34A853} 44.57 ± 0.15} & {\color[HTML]{EA4335} 14.07 ± 0.24} & {\color[HTML]{4285F4} 23.36 ± 0.07} \\
             & PENCIL (Ours)              & 13.13 ± 0.53                        & 16.80 ± 1.43                        & {\color{thirdbest}8.88 ± 0.49}                        & {\color{firstbest} 8.21 ± 0.14}                         & {\color[HTML]{EA4335} 45.43 ± 0.31} & N/A                                 & {\color[HTML]{34A853} 23.43 ± 0.14} \\ \bottomrule
    \end{tabular}%
    }
\end{table*}

\begin{figure*}[t]
    \centering
    \includegraphics[width=0.8\textwidth]{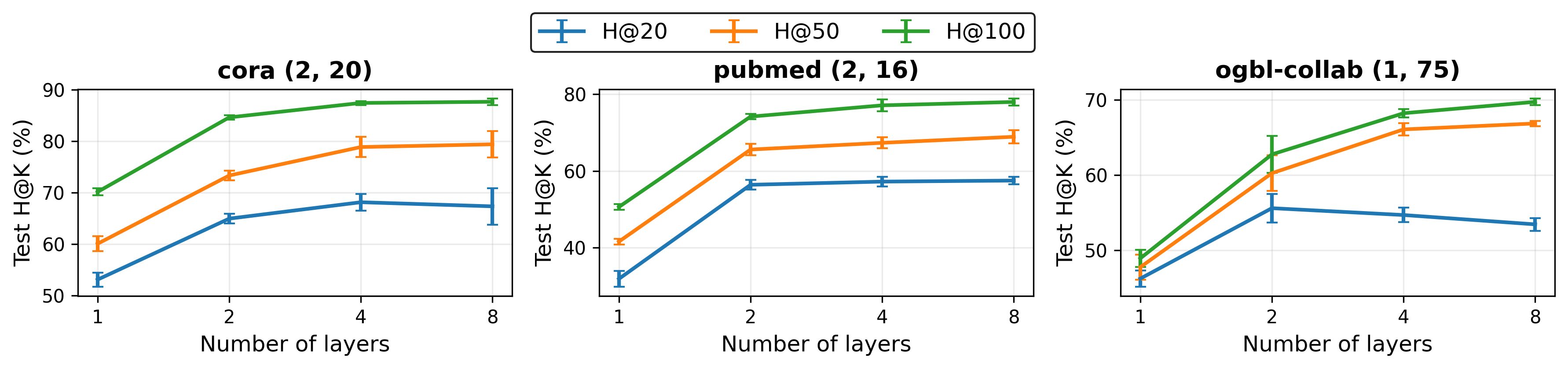}
    \caption{Effect of model depth on different performance metrics across the \texttt{cora}, \texttt{pubmed}, and \texttt{ogbl-collab} datasets.}
    \label{fig:combined_layers_vs_hk}
\end{figure*}

\subsection{Main Results}
Table~\ref{tab:original-setting} and Table~\ref{tab:heart} compare PENCIL with a broad set of baselines, including GCN \cite{kipf2017semisupervised}, SAGE \cite{hamilton2017inductive}, NBFNet \cite{zhu_neural_bellman-ford_2021}, SEAL \cite{zhang_link_2018}, Neo-GNN \cite{yun_neo-gnns_2021}, BUDDY \cite{chamberlain2023graph}, NCN/NCNC \cite{wang2024neural}, LPFormer \cite{shomer_lpformer_2024}, MPLP+ \cite{dong_pure_2024}, and Refined-GAE \cite{ma_reconsidering_2025}. In particular, Table~\ref{tab:original-setting} reports results under the standard single-split protocol: for the Planetoid datasets we use the fixed split from \cite{li_evaluating_2023}, and for the OGB datasets we use the official splits from \cite{hu2020ogb}. We omit \texttt{ogbl-ddi} under the original setting due to the low validation--test correlation noted in \cite{li_evaluating_2023,shomer_lpformer_2024}. Table~\ref{tab:heart} reports results under the HeaRT evaluation protocol \cite{li_evaluating_2023}, in which negative links are curated per positive test example. The two settings share the same training set and differ only in how negatives are constructed for evaluation. We report the dataset-appropriate metrics recommended by \cite{li_evaluating_2023} and \cite{hu2020ogb}. Baseline results are taken from \cite{li_evaluating_2023}, while results for LPFormer, MPLP+, and Refined-GAE are quoted as reported in their respective papers. All remaining experimental details are provided in Appendix~\ref{app:real-world-experiments}.

\textbf{Performance.~} 
While heuristic-informed and ID-based baselines typically dominate, PENCIL challenges this paradigm by delivering robust performance without such priors. In the original setting, PENCIL achieves state-of-the-art results on \texttt{cora} and \texttt{ogbl-ppa}; under HeaRT, it secures top scores on \texttt{ogbl-collab}, \texttt{ogbl-ppa}, and \texttt{ogbl-ddi}. Furthermore, PENCIL demonstrates exceptional stability, consistently exhibiting lower standard deviations than baselines, most notably on \texttt{ogbl-ppa}, where its variance is orders of magnitude smaller ($\pm 0.07$) than that of competing methods. Beyond accuracy and stability, 
PENCIL is highly training-efficient on large-scale benchmarks: it converges on \texttt{ogbl-citation2}, \texttt{ogbl-ddi}, and \texttt{ogbl-ppa} in just 0.5, 8, and 15 epochs, respectively. This is in sharp contrast to pure GNNs, which often require orders of magnitude more training (20--100+ epochs) for the same datasets as reported by \cite{li_evaluating_2023}. This rapid convergence, however, is most pronounced in the data-rich regime. While PENCIL learns quickly from large datasets, it exhibits lower statistical efficiency on small-scale tasks, where it may underperform leading methods. Consequently, the model fails to achieve the same statistical efficiency as seen with larger datasets. This behavior is consistent with the data-intensive requirements of Transformers when recovering graph algorithms \cite{sanford_understanding_2024} and the ``data hunger'' of Vision Transformers (ViT) in Computer Vision, which stems from a lack of translation equivariance \cite{dosovitskiy2021an}. In both cases, the absence of hard-coded inductive biases demands a larger volume of data for generalization. See Appendix~\ref{app:real-world-experiments} for additional training details.

\textbf{Structural Sufficiency.~} Our results indicate that node features are not universally beneficial and can introduce volatility. On Planetoid, specifically \texttt{cora}, adding features degrades performance and significantly inflates variance. In contrast, features provide necessary signal on OGB datasets, yielding major gains on \texttt{ogbl-ppa}. This distinct behavior underscores the complex interplay between feature and structure \cite{coupette2025no,castellana2023investigatinginterplayfeaturesstructures}. Given its exceptional performance using structure alone as reported in Table~\ref{tab:original-setting} and Table~\ref{tab:heart}, PENCIL potentially serves as a critical baseline for this research direction, capable of effectively isolating structural contributions where other models cannot. For further experiments on synthetic graphs with controlled structural and node-feature signals, see Appendix~\ref{app:latent-space-graphs}.

\subsection{Effects of Model Depth}
Standard MPNNs struggle to exploit depth since their receptive field is inherently coupled to graph connectivity, which induces oversmoothing and oversquashing as aggregation expands \cite{li_deeper_2018, oono2020graph, alon2021on}. Conversely, PENCIL effectively harnesses deeper architectures; its utilization of token-level self-attention facilitates dense communication within subgraph instances, mitigating the topological bottlenecks associated with oversquashing \cite{topping2022understanding}. As illustrated in Figure~\ref{fig:combined_layers_vs_hk}, increasing model depth yields consistent improvements in Hits@50 and Hits@100 across the \texttt{cora}, \texttt{pubmed}, and \texttt{ogbl-collab} datasets, exhibiting only mild saturation at greater depths. 
This behavior is markedly different from the MPNN trends documented in Figure~3 of \cite{ma_reconsidering_2025}, where performance typically peaks at shallow depths before dropping sharply as additional layers are introduced.

\section{Related Work} \label{sec:related-work}

\textbf{Link Prediction with GNNs.~} A long line of link predictors augment a base MPNN with auxiliary structural signals. These include overlap-driven/local structure modules (e.g., NCN, Neo-GNN, BUDDY) \cite{wang2024neural,yun_neo-gnns_2021,chamberlain2023graph}, explicit pairwise encodings (e.g., Distance Encoding, DRNL) \cite{li_distance_2020,zhang_link_2018}, and global graph heuristics or per-node ID embeddings (e.g., LPFormer, MPLP, Refined-GAE) \cite{shomer_lpformer_2024,dong_pure_2024,ma_reconsidering_2025}. This stands in contrast to PENCIL, which solely relies on random vector embeddings as node representations (see Section~\ref{sec:plain-transformers-as-link-predictors}). Although PENCIL is formulated as a purely subgraph-based predictor, it could be extended in a scalable hybrid fashion by applying subgraph-based scoring to candidates within the extracted subgraph while using a two-tower retriever for candidates outside it, similarly to ContextGNN \cite{yuan_contextgnn_2025}. We leave the investigation of more scalable approaches for future work.

\textbf{Expressive Powers.~} The expressive power of standard GNNs is upper-bounded by the 1-WL test \cite{xu2018how}. By contrast, Transformers with Laplacian positional encodings are universal approximators and can exceed standard GNN expressivity \cite{kreuzer2021rethinking}. Although PENCIL uses no such PEs, it fits within established expressivity frameworks via relational pooling \cite{murphy_relational_2019,chen_can_2020,zhou_relational_2023}, which symmetrizes a base predictor by averaging over input permutations. Under this lens, PENCIL's randomized vector assignments act as lightweight symmetry breaking, enabling a principled comparison to other link predictors \cite{lachi2025bridgingtheorypracticelink}. More broadly, random features can already be highly expressive \cite{sato_random_features_2021,abboud_surprising_2021}; to our knowledge, PENCIL is the first to formalize link prediction as a randomized Transformer-based predictor and to show that permutation invariance is attainable in distribution.

\textbf{Graph Transformers and PEs/SEs.~}
Prior work on graph Transformers spans a range of choices for injecting structural bias. To name a few, Graph Transformer \cite{dwivedi2021generalization} uses Laplacian eigenvectors as PE, GraphTrans \cite{jain2021representing} stacks a vanilla Transformer on top of GNN outputs, Graphormer \cite{ying_do_transformers_2021} injects multiple structural biases into full-graph attention, and NeuralWalker \cite{chen2025learning} linearizes graphs via random walks. More broadly, because graphs admit no canonical sequence representation, prior work tokenizes structure in various ways---including hops \cite{chen_nagphormer_2023}, eulerization paths \cite{zhao_graphgpt_2025}, edge lists \cite{kim_pure_2022,sanford_understanding_2024}, or nodes \cite{dwivedi_graph_2023,yehudai_depth-width_2025,jain2021representing}. Most of the above graph-Transformer models are developed graph-level (or node-level) prediction, whereas link prediction poses a distinct dyadic setting with different supervision and deployment constraints. To our knowledge, PENCIL is the first to demonstrate that a plain Transformer can serve as a highly effective link predictor under these constraints. By contrast, although LPFormer employs attention for link prediction \cite{shomer_lpformer_2024}, it relies on a PPR-based PE and therefore falls outside our deployment setting.

\section{Conclusion, Limitations and Future Work} 

PENCIL is not simply that a Transformer can be competitive for link prediction, but that we provide, the first systematic explanation of what enables a vanilla bidirectional Transformer to do so without the strong PEs typically required by GTs. PENCIL bridges the gap between high-expressivity link prediction and realistic deployment by utilizing a vanilla Transformer over local subgraphs, entirely bypassing the need for structural encodings or per-node ID dependencies. Our results demonstrate that this architecture extracts richer structural signals than standard GNNs, achieving state-of-the-art performance with significantly fewer parameters and proving that simple, hardware-efficient designs are sufficient for large-scale link prediction. We believe that PENCIL is a principled step toward establishing plain Transformers as a legitimate and theoretically grounded paradigm for link prediction, with broader implications for graph machine learning. Finally, sampling budget serves as a natural knob controlling the expressive power accessible in practice. More broadly, this suggests that advances in hardware and systems for processing larger sampled subgraphs can directly expand the practical capabilities of models like PENCIL.

PENCIL requires a context subgraph for each candidate link; consequently, its GPU compute and memory scale linearly with the number of candidate links. Developing resource-efficient variants---e.g., caching subgraph computations or hybridizing with retrieval-style scoring---remains future work. Moreover, PENCIL's empirical gains are most pronounced on large-scale datasets (Section~\ref{sec:experimental-results}), suggesting that reducing its data requirements through pretraining is a promising avenue. Finally, while our expressivity analysis shows that PENCIL can degenerate to SEAL, a tighter theoretical characterization of full-attention Transformer link predictors is still lacking and would enable clearer comparisons across emerging Transformer-based approaches.

\section*{Acknowledgements}

Quang Truong, Yu Song, and Jiliang Tang are supported by the National Science Foundation (NSF) under grant numbers CNS2321416, IIS2212032, IIS2212144, IIS 2504089, DUE2234015, CNS2246050, DRL2405483 and IOS2035472, the Michigan Department of Agriculture and Rural Development, US Dept of Commerce, Gates Foundation, Amazon Faculty Award, Meta, NVIDIA, Microsoft and SNAP.

\section*{Impact Statement}

This work advances graph machine learning for link prediction by introducing PENCIL, a deployment-oriented Transformer-based link predictor that operates on fixed-budget sampled subgraphs. Potential applications include recommendation, knowledge graph completion, and biological interaction discovery. As with any link prediction system, PENCIL may be used in settings that affect individuals (e.g., recommender systems or social inference), where erroneous, biased, or harmful links could be suggested or amplified. Accordingly, we encourage domain-specific evaluation and responsible deployment practices in high-stakes settings.

\bibliography{main}

@inproceedings{kreuzer2021rethinking,
  title = {Rethinking Graph Transformers with Spectral Attention},
  author = {Devin Kreuzer and Dominique Beaini and William L. Hamilton and Vincent L{\'{e}}tourneau and Prudencio Tossou},
  booktitle = {Advances in Neural Information Processing Systems 34: Annual Conference on Neural Information Processing Systems 2021, NeurIPS 2021, December 6-14, 2021, virtual},
  year = {2021},
  pages = {21618--21629}
}

@inproceedings{li_distance_2020,
  title = {Distance Encoding: Design Provably More Powerful Neural Networks for Graph Representation Learning},
  author = {Pan Li and Yanbang Wang and Hongwei Wang and Jure Leskovec},
  booktitle = {Advances in Neural Information Processing Systems 33: Annual Conference on Neural Information Processing Systems 2020, NeurIPS 2020, December 6-12, 2020, virtual},
  year = {2020}
}

@inproceedings{ying_do_transformers_2021,
author = {Ying, Chengxuan and Cai, Tianle and Luo, Shengjie and Zheng, Shuxin and Ke, Guolin and He, Di and Shen, Yanming and Liu, Tie{-}Yan},
title = {Do transformers really perform bad for graph representation?},
year = {2021},
isbn = {9781713845393},
publisher = {Curran Associates Inc.},
address = {Red Hook, NY, USA},
booktitle = {Proceedings of the 35th International Conference on Neural Information Processing Systems},
articleno = {2212},
numpages = {12},
series = {NIPS '21}
}

@inproceedings{sanford_understanding_2024,
  title = {Understanding Transformer Reasoning Capabilities via Graph Algorithms},
  author = {Clayton Sanford and Bahare Fatemi and Ethan Hall and Anton Tsitsulin and Mehran Kazemi and Jonathan Halcrow and Bryan Perozzi and Vahab Mirrokni},
  booktitle = {Advances in Neural Information Processing Systems 38: Annual Conference on Neural Information Processing Systems 2024, NeurIPS 2024, Vancouver, BC, Canada, December 10 - 15, 2024},
  year = {2024}
}

@inproceedings{ma_graph_2023,
  title = {Graph Inductive Biases in Transformers without Message Passing},
  author = {Liheng Ma and Chen Lin and Derek Lim and Adriana Romero{-}Soriano and Puneet K. Dokania and Mark Coates and Philip H. S. Torr and Ser{-}Nam Lim},
  booktitle = {International Conference on Machine Learning, {ICML} 2023, 23-29 July 2023, Honolulu, Hawaii, {USA}},
  year = {2023},
  volume = {202},
  pages = {23321--23337}
}

@inproceedings{chamberlain2023graph,
  title = {Graph Neural Networks for Link Prediction with Subgraph Sketching},
  author = {Benjamin Paul Chamberlain and Sergey Shirobokov and Emanuele Rossi and Fabrizio Frasca and Thomas Markovich and Nils Yannick Hammerla and Michael M. Bronstein and Max Hansmire},
  booktitle = {The Eleventh International Conference on Learning Representations, {ICLR} 2023, Kigali, Rwanda, May 1-5, 2023},
  year = {2023}
}

@inproceedings{wang2024neural,
  title = {Neural Common Neighbor with Completion for Link Prediction},
  author = {Xiyuan Wang and Haotong Yang and Muhan Zhang},
  booktitle = {The Twelfth International Conference on Learning Representations, {ICLR} 2024, Vienna, Austria, May 7-11, 2024},
  year = {2024}
}

@inproceedings{zhu_neural_bellman-ford_2021,
  title = {{Neural} {Bellman-Ford} Networks: {A} General Graph Neural Network Framework for Link Prediction},
  author = {Zhaocheng Zhu and Zuobai Zhang and Louis{-}Pascal A. C. Xhonneux and Jian Tang},
  booktitle = {Advances in Neural Information Processing Systems 34: Annual Conference on Neural Information Processing Systems 2021, NeurIPS 2021, December 6-14, 2021, virtual},
  year = {2021},
  pages = {29476--29490}
}

@inproceedings{zhang_link_2018,
author = {Zhang, Muhan and Chen, Yixin},
title = {Link prediction based on graph neural networks},
year = {2018},
publisher = {Curran Associates Inc.},
address = {Red Hook, NY, USA},
booktitle = {Proceedings of the 32nd International Conference on Neural Information Processing Systems},
pages = {5171–5181},
numpages = {11},
location = {Montr\'{e}al, Canada},
series = {NIPS'18}
}

@inproceedings{zhang_labeling_2022,
  author       = {Muhan Zhang and Pan Li and Yinglong Xia and Kai Wang and Long Jin},
  editor       = {Marc'Aurelio Ranzato and
                  Alina Beygelzimer and
                  Yann N. Dauphin and
                  Percy Liang and
                  Jennifer Wortman Vaughan},
  title        = {Labeling Trick: {A} Theory of Using Graph Neural Networks for Multi-Node
                  Representation Learning},
  booktitle    = {Advances in Neural Information Processing Systems 34: Annual Conference
                  on Neural Information Processing Systems 2021, NeurIPS 2021, December
                  6-14, 2021, virtual},
  pages        = {9061--9073},
  year         = {2021},
  bibsource    = {dblp computer science bibliography, https://dblp.org}
}

@inproceedings{li_evaluating_2023,
 author = {Li, Juanhui and Shomer, Harry and Mao, Haitao and Zeng, Shenglai and Ma, Yao and Shah, Neil and Tang, Jiliang and Yin, Dawei},
 booktitle = {Advances in Neural Information Processing Systems},
 doi = {10.52202/075280-0169},
 editor = {A. Oh and T. Naumann and A. Globerson and K. Saenko and M. Hardt and S. Levine},
 pages = {3853--3866},
 publisher = {Curran Associates, Inc.},
 title = {Evaluating Graph Neural Networks for Link Prediction: Current Pitfalls and New Benchmarking},
 volume = {36},
 year = {2023}
}

@inproceedings{
yehudai_depth-width_2025,
title={Depth-Width Tradeoffs for Transformers on Graph Tasks},
author={Gilad Yehudai and Clayton Sanford and Maya Bechler-Speicher and Orr Fischer and Ran Gilad-Bachrach and Amir Globerson},
booktitle={The Thirty-ninth Annual Conference on Neural Information Processing Systems},
year={2025},
url={https://openreview.net/forum?id=A2pmNL7L1E}
}

@inproceedings{chen_nagphormer_2023,
 author = {Jinsong Chen and Kaiyuan Gao and Gaichao Li and Kun He},
 bibsource = {dblp computer science bibliography, https://dblp.org},
 booktitle = {The Eleventh International Conference on Learning Representations,
{ICLR} 2023, Kigali, Rwanda, May 1-5, 2023},
 publisher = {OpenReview.net},
 title = {{NAGphormer}: {A} Tokenized Graph Transformer for Node Classification
in Large Graphs},
 url = {https://openreview.net/pdf?id=8KYeilT3Ow},
 year = {2023}
}

@inproceedings{kim_pure_2022,
 author = {Jinwoo Kim and Dat Nguyen and Seonwoo Min and Sungjun Cho and Moontae Lee and Honglak Lee and Seunghoon Hong},
 bibsource = {dblp computer science bibliography, https://dblp.org},
 booktitle = {Advances in Neural Information Processing Systems 35: Annual Conference
on Neural Information Processing Systems 2022, NeurIPS 2022, New Orleans,
LA, USA, November 28 - December 9, 2022},
 editor = {Sanmi Koyejo and
S. Mohamed and
A. Agarwal and
Danielle Belgrave and
K. Cho and
A. Oh},
 title = {Pure Transformers are Powerful Graph Learners},
 year = {2022}
}

@inproceedings{
    zhao_graphgpt_2025,
    title={Graph{GPT}: Generative Pre-trained Graph {Eulerian} Transformer},
    author={Qifang Zhao and Weidong Ren and Tianyu Li and Hong Liu and Xingsheng He and Xiaoxiao Xu},
    booktitle={Forty-second International Conference on Machine Learning},
    year={2025},
    url={https://openreview.net/forum?id=4RdzeucFmW}
}

@article{ma_plain_2025,
 author = {Ma, Liheng and Pal, Soumyasundar and Zhang, Yingxue and Torr, Philip H. S. and Coates, Mark},
 journal = {ArXiv preprint},
 title = {Plain Transformers Can be Powerful Graph Learners},
 url = {https://arxiv.org/abs/2504.12588},
 volume = {abs/2504.12588},
 year = {2025}
}

@inproceedings{shomer_lpformer_2024,
 author = {Harry Shomer and Yao Ma and Haitao Mao and Juanhui Li and Bo Wu and Jiliang Tang},
 bibsource = {dblp computer science bibliography, https://dblp.org},
 booktitle = {Proceedings of the 30th {ACM} {SIGKDD} Conference on Knowledge Discovery
and Data Mining, {KDD} 2024, Barcelona, Spain, August 25-29, 2024},
 doi = {10.1145/3637528.3672025},
 editor = {Ricardo Baeza{-}Yates and
Francesco Bonchi},
 pages = {2686--2698},
 publisher = {{ACM}},
 title = {{LPFormer}: An Adaptive Graph Transformer for Link Prediction},
 url = {https://doi.org/10.1145/3637528.3672025},
 year = {2024}
}

@article{dwivedi_graph_2023,
 author = {Dwivedi, Vijay Prakash and Liu, Yozen and Luu, Anh Tuan and Bresson, Xavier and Shah, Neil and Zhao, Tong},
 journal = {ArXiv preprint},
 title = {Graph Transformers for Large Graphs},
 url = {https://arxiv.org/abs/2312.11109},
 volume = {abs/2312.11109},
 year = {2023}
}

@inproceedings{
kim_revisiting_2025,
title={Revisiting Random Walks for Learning on Graphs},
author={Jinwoo Kim and Olga Zaghen and Ayhan Suleymanzade and Youngmin Ryou and Seunghoon Hong},
booktitle={The Thirteenth International Conference on Learning Representations},
year={2025},
url={https://openreview.net/forum?id=SG1R2H3fa1}
}

@article{
muller_attending_2024,
title={Attending to Graph Transformers},
author={Luis M{\"u}ller and Mikhail Galkin and Christopher Morris and Ladislav Ramp{\'a}{\v{s}}ek},
journal={Transactions on Machine Learning Research},
issn={2835-8856},
year={2024},
url={https://openreview.net/forum?id=HhbqHBBrfZ},
note={}
}

@inproceedings{ma_reconsidering_2025,
author = {Ma, Weishuo and Wang, Yanbo and Wang, Xiyuan and Zhang, Muhan},
title = {Reconsidering the Performance of {GAE} in Link Prediction},
year = {2025},
isbn = {9798400720406},
publisher = {Association for Computing Machinery},
address = {New York, NY, USA},
url = {https://doi.org/10.1145/3746252.3761108},
doi = {10.1145/3746252.3761108},
booktitle = {Proceedings of the 34th ACM International Conference on Information and Knowledge Management},
pages = {2052–2062},
numpages = {11},
location = {Seoul, Republic of Korea},
series = {CIKM '25}
}

@inproceedings{dong_pure_2024,
 author = {Kaiwen Dong and Zhichun Guo and Nitesh V. Chawla},
 bibsource = {dblp computer science bibliography, https://dblp.org},
 booktitle = {Advances in Neural Information Processing Systems 38: Annual Conference
on Neural Information Processing Systems 2024, NeurIPS 2024, Vancouver,
BC, Canada, December 10 - 15, 2024},
 editor = {Amir Globersons and
Lester Mackey and
Danielle Belgrave and
Angela Fan and
Ulrich Paquet and
Jakub M. Tomczak and
Cheng Zhang},
 title = {Pure Message Passing Can Estimate Common Neighbor for Link Prediction},
 year = {2024}
}

@inproceedings{sato_random_features_2021,
  author    = {Ryoma Sato and Makoto Yamada and Hisashi Kashima},
  title     = {Random Features Strengthen Graph Neural Networks},
  booktitle = {Proceedings of the 2021 {SIAM} International Conference on Data Mining, {SDM}},
  year      = {2021},
}

@inproceedings{abboud_surprising_2021,
 author = {Ralph Abboud and {\.I}smail {\.I}lkan Ceylan and Martin Grohe and Thomas Lukasiewicz},
 bibsource = {dblp computer science bibliography, https://dblp.org},
 booktitle = {Proceedings of the Thirtieth International Joint Conference on Artificial
Intelligence, {IJCAI} 2021, Virtual Event / Montreal, Canada, 19-27
August 2021},
 doi = {10.24963/IJCAI.2021/291},
 editor = {Zhi{-}Hua Zhou},
 pages = {2112--2118},
 publisher = {ijcai.org},
 title = {The Surprising Power of Graph Neural Networks with Random Node Initialization},
 url = {https://doi.org/10.24963/ijcai.2021/291},
 year = {2021}
}

@article{liang_can_2025,
 author = {Liang, Shuming and Ding, Yu and Li, Zhidong and Liang, Bin and Zhang, Siqi and Wang, Yang and Chen, Fang},
 journal = {ArXiv preprint},
 title = {Can {GNNs} Learn Link Heuristics? {A} Concise Review and Evaluation of Link Prediction Methods},
 url = {https://arxiv.org/abs/2411.14711},
 volume = {abs/2411.14711},
 year = {2024}
}

@inproceedings{murphy_relational_2019,
 author = {Ryan L. Murphy and Balasubramaniam Srinivasan and Vinayak A. Rao and Bruno Ribeiro},
 bibsource = {dblp computer science bibliography, https://dblp.org},
 booktitle = {Proceedings of the 36th International Conference on Machine Learning,
{ICML} 2019, 9-15 June 2019, Long Beach, California, {USA}},
 editor = {Kamalika Chaudhuri and
Ruslan Salakhutdinov},
 pages = {4663--4673},
 publisher = {{PMLR}},
 series = {Proceedings of Machine Learning Research},
 title = {Relational Pooling for Graph Representations},
 url = {http://proceedings.mlr.press/v97/murphy19a.html},
 volume = {97},
 year = {2019}
}

@inproceedings{zhou_relational_2023,
 author = {Cai Zhou and Xiyuan Wang and Muhan Zhang},
 bibsource = {dblp computer science bibliography, https://dblp.org},
 booktitle = {International Conference on Machine Learning, {ICML} 2023, 23-29 July
2023, Honolulu, Hawaii, {USA}},
 editor = {Andreas Krause and
Emma Brunskill and
Kyunghyun Cho and
Barbara Engelhardt and
Sivan Sabato and
Jonathan Scarlett},
 pages = {42742--42768},
 publisher = {{PMLR}},
 series = {Proceedings of Machine Learning Research},
 title = {From Relational Pooling to Subgraph {GNNs}: {A} Universal Framework
for More Expressive Graph Neural Networks},
 url = {https://proceedings.mlr.press/v202/zhou23n.html},
 volume = {202},
 year = {2023}
}

@article{baras_path_2010,
 author = {John S. Baras and George Theodorakopoulos},
 journal = {Synthesis Lectures on Communication Networks},
 pages = {1-77},
 title = {Path Problems in Networks},
 volume = {3},
 year = {2010}
}

@inproceedings{chen_can_2020,
 author = {Zhengdao Chen and Lei Chen and Soledad Villar and Joan Bruna},
 bibsource = {dblp computer science bibliography, https://dblp.org},
 booktitle = {Advances in Neural Information Processing Systems 33: Annual Conference
on Neural Information Processing Systems 2020, NeurIPS 2020, December
6-12, 2020, virtual},
 editor = {Hugo Larochelle and
Marc'Aurelio Ranzato and
Raia Hadsell and
Maria{-}Florina Balcan and
Hsuan{-}Tien Lin},
 title = {Can Graph Neural Networks Count Substructures?},
 year = {2020}
}

@inproceedings{
    lachi2025bridgingtheorypracticelink,
    title={Bridging Theory and Practice in Link Representation with Graph Neural Networks},
    author={Veronica Lachi and Francesco Ferrini and Antonio Longa and Bruno Lepri and Andrea Passerini and Manfred Jaeger},
    booktitle={The Thirty-ninth Annual Conference on Neural Information Processing Systems},
    year={2025},
    url={https://openreview.net/forum?id=WYnvP3DePZ}
}

@article{Abbas2021ApplicationON,
  title = {Application of network link prediction in drug discovery},
  author = {Khushnood Abbas and Alireza Abbasi and Shi Dong and Niu Ling and Laihang Yu and Bolun Chen and Shi{-}Min Cai and Qambar Hasan},
  journal = {{BMC} Bioinform.},
  year = {2021},
  volume = {22},
  pages = {187},
  doi = {10.1186/S12859-021-04082-Y}
}

@inproceedings{huang_link_2005,
author = {Huang, Zan and Li, Xin and Chen, Hsinchun},
title = {Link prediction approach to collaborative filtering},
year = {2005},
isbn = {1581138768},
publisher = {Association for Computing Machinery},
address = {New York, NY, USA},
url = {https://doi.org/10.1145/1065385.1065415},
doi = {10.1145/1065385.1065415},
booktitle = {Proceedings of the 5th ACM/IEEE-CS Joint Conference on Digital Libraries},
pages = {141–142},
numpages = {2},
location = {Denver, CO, USA},
series = {JCDL '05}
}

@inproceedings{yun_neo-gnns_2021,
author = {Yun, Seongjun and Kim, Seoyoon and Lee, Junhyun and Kang, Jaewoo and Kim, Hyunwoo J.},
title = {{Neo-GNNs}: {N}eighborhood overlap-aware graph neural networks for link prediction},
year = {2021},
isbn = {9781713845393},
publisher = {Curran Associates Inc.},
address = {Red Hook, NY, USA},
booktitle = {Proceedings of the 35th International Conference on Neural Information Processing Systems},
articleno = {1048},
numpages = {12},
series = {NIPS '21}
}

@article{brin_anatomy_1998,
title = {The anatomy of a large-scale hypertextual Web search engine},
journal = {Computer Networks and ISDN Systems},
volume = {30},
number = {1},
pages = {107-117},
year = {1998},
note = {Proceedings of the Seventh International World Wide Web Conference},
issn = {0169-7552},
doi = {https://doi.org/10.1016/S0169-7552(98)00110-X},
url = {https://www.sciencedirect.com/science/article/pii/S016975529800110X},
author = {Sergey Brin and Lawrence Page}
}

@inproceedings{fey_rdl_2024,
  title = 	 {Position: Relational Deep Learning - Graph Representation Learning on Relational Databases},
  author =       {Fey, Matthias and Hu, Weihua and Huang, Kexin and Lenssen, Jan Eric and Ranjan, Rishabh and Robinson, Joshua and Ying, Rex and You, Jiaxuan and Leskovec, Jure},
  booktitle = 	 {Proceedings of the 41st International Conference on Machine Learning},
  pages = 	 {13592--13607},
  year = 	 {2024},
  editor = 	 {Salakhutdinov, Ruslan and Kolter, Zico and Heller, Katherine and Weller, Adrian and Oliver, Nuria and Scarlett, Jonathan and Berkenkamp, Felix},
  volume = 	 {235},
  series = 	 {Proceedings of Machine Learning Research},
  month = 	 {21--27 Jul},
  publisher =    {PMLR},
  url = 	 {https://proceedings.mlr.press/v235/fey24a.html}
}

@inproceedings{zhao_gigl_2025,
  title = {{GiGL}: Large-Scale Graph Neural Networks at {Snapchat}},
  author = {Tong Zhao and Yozen Liu and Matthew Kolodner and Kyle Montemayor and Elham Ghazizadeh and Ankit Batra and Zihao Fan and Xiaobin Gao and Xuan Guo and Jiwen Ren and Serim Park and Peicheng Yu and Jun Yu and Shubham Vij and Neil Shah},
  booktitle = {Proceedings of the 31st {ACM} {SIGKDD} Conference on Knowledge Discovery and Data Mining, V.2, {KDD} 2025, Toronto ON, Canada, August 3-7, 2025},
  year = {2025},
  pages = {5225--5236},
  doi = {10.1145/3711896.3737229}
}

@article{
longa2023graph,
title={Graph Neural Networks for Temporal Graphs: State of the Art, Open Challenges, and Opportunities},
author={Antonio Longa and Veronica Lachi and Gabriele Santin and Monica Bianchini and Bruno Lepri and Pietro Lio and Franco Scarselli and Andrea Passerini},
journal={Transactions on Machine Learning Research},
issn={2835-8856},
year={2023},
url={https://openreview.net/forum?id=pHCdMat0gI},
note={}
}

@inproceedings{dao2022flashattention,
  title={Flash{A}ttention: Fast and Memory-Efficient Exact Attention with {IO}-Awareness},
  author={Dao, Tri and Fu, Daniel Y. and Ermon, Stefano and Rudra, Atri and R{\'e}, Christopher},
  booktitle={Advances in Neural Information Processing Systems (NeurIPS)},
  year={2022}
}

@inproceedings{devlin-etal-2019-bert,
    title = "{BERT}: Pre-training of Deep Bidirectional Transformers for Language Understanding",
    author = "Devlin, Jacob  and
      Chang, Ming-Wei  and
      Lee, Kenton  and
      Toutanova, Kristina",
    editor = "Burstein, Jill  and
      Doran, Christy  and
      Solorio, Thamar",
    booktitle = "Proceedings of the 2019 Conference of the North {A}merican Chapter of the Association for Computational Linguistics: Human Language Technologies, Volume 1 (Long and Short Papers)",
    month = jun,
    year = "2019",
    address = "Minneapolis, Minnesota",
    publisher = "Association for Computational Linguistics",
    url = "https://aclanthology.org/N19-1423/",
    doi = "10.18653/v1/N19-1423",
    pages = "4171--4186"
}

@inproceedings{vaswani2017attention,
author = {Vaswani, Ashish and Shazeer, Noam and Parmar, Niki and Uszkoreit, Jakob and Jones, Llion and Gomez, Aidan N. and Kaiser, \L{}ukasz and Polosukhin, Illia},
title = {Attention is all you need},
year = {2017},
isbn = {9781510860964},
publisher = {Curran Associates Inc.},
address = {Red Hook, NY, USA},
booktitle = {Proceedings of the 31st International Conference on Neural Information Processing Systems},
pages = {6000–6010},
numpages = {11},
location = {Long Beach, California, USA},
series = {NIPS'17}
}

@inproceedings{
kipf2017semisupervised,
title={Semi-Supervised Classification with Graph Convolutional Networks},
author={Thomas N. Kipf and Max Welling},
booktitle={International Conference on Learning Representations},
year={2017},
url={https://openreview.net/forum?id=SJU4ayYgl}
}

@inproceedings{hamilton2017inductive,
     author = {Hamilton, William L. and Ying, Rex and Leskovec, Jure},
     title = {Inductive Representation Learning on Large Graphs},
     booktitle = {NIPS},
     year = {2017}
   }

@inproceedings{
veličković2018graph,
title={Graph Attention Networks},
author={Petar Veličković and Guillem Cucurull and Arantxa Casanova and Adriana Romero and Pietro Liò and Yoshua Bengio},
booktitle={International Conference on Learning Representations},
year={2018},
url={https://openreview.net/forum?id=rJXMpikCZ},
}

@article{hu2020ogb,
  title={{O}pen {G}raph {B}enchmark: Datasets for Machine Learning on Graphs},
  author={Hu, Weihua and Fey, Matthias and Zitnik, Marinka and Dong, Yuxiao and Ren, Hongyu and Liu, Bowen and Catasta, Michele and Leskovec, Jure},
  journal={arXiv preprint arXiv:2005.00687},
  year={2020}
}

@article{dwivedi2021generalization,
  title={A Generalization of Transformer Networks to Graphs},
  author={Dwivedi, Vijay Prakash and Bresson, Xavier},
  journal={AAAI Workshop on Deep Learning on Graphs: Methods and Applications},
  year={2021}
}

@inproceedings{yang_revisiting_2016,
author = {Yang, Zhilin and Cohen, William W. and Salakhutdinov, Ruslan},
title = {Revisiting semi-supervised learning with graph embeddings},
year = {2016},
publisher = {JMLR.org},
booktitle = {Proceedings of the 33rd International Conference on International Conference on Machine Learning - Volume 48},
pages = {40–48},
numpages = {9},
location = {New York, NY, USA},
series = {ICML'16}
}

@article{newman_clustering_2001,
  title = {Clustering and preferential attachment in growing networks},
  author = {Newman, M. E. J.},
  journal = {Phys. Rev. E},
  volume = {64},
  issue = {2},
  pages = {025102},
  numpages = {4},
  year = {2001},
  month = {Jul},
  publisher = {American Physical Society},
  doi = {10.1103/PhysRevE.64.025102},
  url = {https://link.aps.org/doi/10.1103/PhysRevE.64.025102}
}

@article{adamic_friends_2003,
  title = {Friends and neighbors on the Web},
  author = {Lada A. Adamic and Eytan Adar},
  journal = {Soc. Networks},
  year = {2003},
  volume = {25},
  pages = {211--230},
  doi = {10.1016/S0378-8733(03)00009-1}
}

@article{zhou_predicting_2009,
author = {Zhou, Tao and Lü, Linyuan and Zhang, Yi-Cheng},
year = {2009},
month = {10},
pages = {623-630},
title = {Predicting Missing Links via Local Information},
volume = {71},
journal = {The European Physical Journal B - Condensed Matter and Complex Systems},
doi = {10.1140/epjb/e2009-00335-8}
}

@article{katz_new_1953,
  title   = {A New Status Index Derived from Sociometric Analysis},
  volume  = {18},
  doi     = {10.1007/BF02289026},
  number  = {1},
  journal = {Psychometrika},
  author  = {Katz, Leo},
  year    = {1953},
  pages   = {39-43}
}

@inproceedings{you2021identity,
  title={Identity-aware graph neural networks},
  author={You, Jiaxuan and Gomes-Selman, Jonathan M and Ying, Rex and Leskovec, Jure},
  booktitle={Proceedings of the AAAI conference on artificial intelligence},
  pages={10737--10745},
  year={2021}
}

@inproceedings{
robinson2024relbench,
title={{RelBench}: A Benchmark for Deep Learning on Relational Databases},
author={Joshua Robinson and Rishabh Ranjan and Weihua Hu and Kexin Huang and Jiaqi Han and Alejandro Dobles and Matthias Fey and Jan Eric Lenssen and Yiwen Yuan and Zecheng Zhang and Xinwei He and Jure Leskovec},
booktitle={The Thirty-eight Conference on Neural Information Processing Systems Datasets and Benchmarks Track},
year={2024},
url={https://openreview.net/forum?id=WEFxOm3Aez}
}

@inproceedings{yuan_contextgnn_2025,
title = {{ContextGNN}: Beyond Two-Tower Recommendation Systems},
author = {Yiwen Yuan and Zecheng Zhang and Xinwei He and Akihiro Nitta and Weihua Hu and Manan Shah and Bla{\v{z}} Stojanovi{\v{c}} and Shenyang Huang and Jan Eric Lenssen and Jure Leskovec and Matthias Fey},
booktitle = {The Thirteenth International Conference on Learning Representations (ICLR)},
year = {2025}
}

@inproceedings{truong2024weisfeiler,
  title={{Weisfeiler} and {Lehman} Go Paths: Learning Topological Features via Path Complexes},
  author={Truong, Quang and Chin, Peter},
  booktitle={Proceedings of the AAAI Conference on Artificial Intelligence},
  volume={38},
  pages={15382--15391},
  year={2024}
}

@inproceedings{zaheer_deep_2017,
author = {Zaheer, Manzil and Kottur, Satwik and Ravanbhakhsh, Siamak and P\'{o}czos, Barnab\'{a}s and Salakhutdinov, Ruslan and Smola, Alexander J},
title = {Deep Sets},
year = {2017},
isbn = {9781510860964},
publisher = {Curran Associates Inc.},
address = {Red Hook, NY, USA},
booktitle = {Proceedings of the 31st International Conference on Neural Information Processing Systems},
pages = {3394–3404},
numpages = {11},
location = {Long Beach, California, USA},
series = {NIPS'17}
}

@inproceedings{
maron2018invariant,
title={Invariant and Equivariant Graph Networks},
author={Haggai Maron and Heli Ben-Hamu and Nadav Shamir and Yaron Lipman},
booktitle={International Conference on Learning Representations},
year={2019},
url={https://openreview.net/forum?id=Syx72jC9tm},
}

@inproceedings{bodnar_weisfeiler_2021,
  title = 	 {Weisfeiler and {Lehman} Go Topological: Message Passing Simplicial Networks},
  author =       {Bodnar, Cristian and Frasca, Fabrizio and Wang, Yuguang and Otter, Nina and Montufar, Guido F and Li{\'o}, Pietro and Bronstein, Michael},
  booktitle = 	 {Proceedings of the 38th International Conference on Machine Learning},
  pages = 	 {1026--1037},
  year = 	 {2021},
  editor = 	 {Meila, Marina and Zhang, Tong},
  volume = 	 {139},
  series = 	 {Proceedings of Machine Learning Research},
  month = 	 {18--24 Jul},
  publisher =    {PMLR},
}

@article{welch_lower_1974,
  author={Welch, L.},
  journal={IEEE Transactions on Information Theory}, 
  title={Lower bounds on the maximum cross correlation of signals (Corresp.)}, 
  year={1974},
  volume={20},
  number={3},
  pages={397-399},
  doi={10.1109/TIT.1974.1055219}}

@ARTICLE{donoho_uncertainty_2001,
  author={Donoho, D.L. and Huo, X.},
  journal={IEEE Transactions on Information Theory}, 
  title={Uncertainty principles and ideal atomic decomposition}, 
  year={2001},
  volume={47},
  number={7},
  pages={2845-2862},
  doi={10.1109/18.959265}}

@inproceedings{
zeng2021decoupling,
title={Decoupling the Depth and Scope of Graph Neural Networks},
author={Hanqing Zeng and Muhan Zhang and Yinglong Xia and Ajitesh Srivastava and Andrey Malevich and Rajgopal Kannan and Viktor Prasanna and Long Jin and Ren Chen},
booktitle={Advances in Neural Information Processing Systems},
editor={A. Beygelzimer and Y. Dauphin and P. Liang and J. Wortman Vaughan},
year={2021},
url={https://openreview.net/forum?id=_IY3_4psXuf}
}

@inproceedings{wolf-etal-2020-transformers,
    title = "Transformers: State-of-the-Art Natural Language Processing",
    author = "Thomas Wolf and Lysandre Debut and Victor Sanh and Julien Chaumond and Clement Delangue and Anthony Moi and Pierric Cistac and Tim Rault and Rémi Louf and Morgan Funtowicz and Joe Davison and Sam Shleifer and Patrick von Platen and Clara Ma and Yacine Jernite and Julien Plu and Canwen Xu and Teven Le Scao and Sylvain Gugger and Mariama Drame and Quentin Lhoest and Alexander M. Rush",
    booktitle = "Proceedings of the 2020 Conference on Empirical Methods in Natural Language Processing: System Demonstrations",
    month = oct,
    year = "2020",
    address = "Online",
    publisher = "Association for Computational Linguistics",
    url = "https://www.aclweb.org/anthology/2020.emnlp-demos.6",
    pages = "38--45"
}

@inproceedings{
coupette2025no,
title={No Metric to Rule Them All: Toward Principled Evaluations of Graph-Learning Datasets},
author={Corinna Coupette and Jeremy Wayland and Emily Simons and Bastian Rieck},
booktitle={Forty-second International Conference on Machine Learning},
year={2025},
url={https://openreview.net/forum?id=XbmBNwrfG5}
}

@misc{castellana2023investigatinginterplayfeaturesstructures,
      title={Investigating the Interplay between Features and Structures in Graph Learning}, 
      author={Daniele Castellana and Federico Errica},
      year={2023},
      eprint={2308.09570},
      archivePrefix={arXiv},
      primaryClass={cs.LG},
      url={https://arxiv.org/abs/2308.09570}, 
}

@inproceedings{
    alon2021on,
    title={On the Bottleneck of Graph Neural Networks and its Practical Implications},
    author={Uri Alon and Eran Yahav},
    booktitle={International Conference on Learning Representations},
    year={2021},
    url={https://openreview.net/forum?id=i80OPhOCVH2}
}

@article{li_deeper_2018,
  title        = {Deeper Insights Into Graph Convolutional Networks for Semi-Supervised Learning},
  volume       = {32},
  url          = {https://ojs.aaai.org/index.php/AAAI/article/view/11604},
  doi          = {10.1609/aaai.v32i1.11604},
  number       = {1},
  journal      = {Proceedings of the AAAI Conference on Artificial Intelligence},
  author       = {Li, Qimai and Han, Zhichao and Wu, Xiao-ming},
  year         = {2018},
  month        = {Apr.}
}

@inproceedings{
oono2020graph,
title={Graph Neural Networks Exponentially Lose Expressive Power for Node Classification},
author={Kenta Oono and Taiji Suzuki},
booktitle={International Conference on Learning Representations},
year={2020},
url={https://openreview.net/forum?id=S1ldO2EFPr}
}

@inproceedings{
topping2022understanding,
title={Understanding over-squashing and bottlenecks on graphs via curvature},
author={Jake Topping and Francesco Di Giovanni and Benjamin Paul Chamberlain and Xiaowen Dong and Michael M. Bronstein},
booktitle={International Conference on Learning Representations},
year={2022},
url={https://openreview.net/forum?id=7UmjRGzp-A}
}

@article{rampasek2022GPS,
  title={Recipe for a General, Powerful, Scalable Graph Transformer}, 
  author={Ladislav Ramp\'{a}\v{s}ek and Mikhail Galkin and Vijay Prakash Dwivedi and Anh Tuan Luu and Guy Wolf and Dominique Beaini},
  journal={Advances in Neural Information Processing Systems},
  volume={35},
  year={2022}
}

@inproceedings{
jain2021representing,
title={Representing Long-Range Context for Graph Neural Networks with Global Attention},
author={Paras Jain and Zhanghao Wu and Matthew A. Wright and Azalia Mirhoseini and Joseph E. Gonzalez and Ion Stoica},
booktitle={Advances in Neural Information Processing Systems},
editor={A. Beygelzimer and Y. Dauphin and P. Liang and J. Wortman Vaughan},
year={2021},
url={https://openreview.net/forum?id=nYz2_BbZnYk}
}

@inproceedings{
chen2025learning,
title={Learning Long Range Dependencies on Graphs via Random Walks},
author={Dexiong Chen and Till Hendrik Schulz and Karsten Borgwardt},
booktitle={The Thirteenth International Conference on Learning Representations},
year={2025},
url={https://openreview.net/forum?id=kJ5H7oGT2M}
}

@misc{dwivedi2025relationalgraphtransformer,
      title={Relational Graph Transformer}, 
      author={Vijay Prakash Dwivedi and Sri Jaladi and Yangyi Shen and Federico López and Charilaos I. Kanatsoulis and Rishi Puri and Matthias Fey and Jure Leskovec},
      year={2025},
      eprint={2505.10960},
      archivePrefix={arXiv},
      primaryClass={cs.LG},
      url={https://arxiv.org/abs/2505.10960}, 
}

@inproceedings{chen2022structure,
  title={Structure-aware transformer for graph representation learning},
  author={Chen, Dexiong and O’Bray, Leslie and Borgwardt, Karsten},
  booktitle={International conference on machine learning},
  pages={3469--3489},
  year={2022},
  organization={PMLR}
}

@inproceedings{
dosovitskiy2021an,
title={An Image is Worth 16x16 Words: Transformers for Image Recognition at Scale},
author={Alexey Dosovitskiy and Lucas Beyer and Alexander Kolesnikov and Dirk Weissenborn and Xiaohua Zhai and Thomas Unterthiner and Mostafa Dehghani and Matthias Minderer and Georg Heigold and Sylvain Gelly and Jakob Uszkoreit and Neil Houlsby},
booktitle={International Conference on Learning Representations},
year={2021},
url={https://openreview.net/forum?id=YicbFdNTTy}
}

@inproceedings{
xu2018how,
title={How Powerful are Graph Neural Networks?},
author={Keyulu Xu and Weihua Hu and Jure Leskovec and Stefanie Jegelka},
booktitle={International Conference on Learning Representations},
year={2019},
url={https://openreview.net/forum?id=ryGs6iA5Km},
}

@inproceedings{
choromanski2021rethinking,
title={Rethinking Attention with {Performers}},
author={Krzysztof Marcin Choromanski and Valerii Likhosherstov and David Dohan and Xingyou Song and Andreea Gane and Tamas Sarlos and Peter Hawkins and Jared Quincy Davis and Afroz Mohiuddin and Lukasz Kaiser and David Benjamin Belanger and Lucy J Colwell and Adrian Weller},
booktitle={International Conference on Learning Representations},
year={2021},
url={https://openreview.net/forum?id=Ua6zuk0WRH}
}

@inproceedings{fey2019pyg,
  title={Fast Graph Representation Learning with {PyTorch} {Geometric}},
  author={Fey, Matthias and Lenssen, Jan E.},
  booktitle={ICLR Workshop on Representation Learning on Graphs and Manifolds},
  year={2019},
}

@inproceedings{
fey2025pyg,
title={{PyG} 2.0: Scalable Learning on Real World Graphs},
author={Matthias Fey and Jinu Sunil and Akihiro Nitta and Rishi Puri and Manan Shah and Bla{\v{z}} Stojanovi{\v{c}} and Ramona Bendias and Alexandria Barghi and Vid Kocijan and Zecheng Zhang and Xinwei He and Jan Eric Lenssen and Jure Leskovec},
booktitle={Temporal Graph Learning Workshop @ KDD 2025},
year={2025},
url={https://openreview.net/forum?id=DHHLkQvWqs}
}

@inbook{paszke2019pytorch,
author = {Paszke, Adam and Gross, Sam and Massa, Francisco and Lerer, Adam and Bradbury, James and Chanan, Gregory and Killeen, Trevor and Lin, Zeming and Gimelshein, Natalia and Antiga, Luca and Desmaison, Alban and K\"{o}pf, Andreas and Yang, Edward and DeVito, Zach and Raison, Martin and Tejani, Alykhan and Chilamkurthy, Sasank and Steiner, Benoit and Fang, Lu and Bai, Junjie and Chintala, Soumith},
title = {PyTorch: {A}n imperative style, high-performance deep learning library},
year = {2019},
publisher = {Curran Associates Inc.},
address = {Red Hook, NY, USA},
booktitle = {Proceedings of the 33rd International Conference on Neural Information Processing Systems},
articleno = {721},
numpages = {12}
}

@inproceedings{
zhao2022from,
title={From Stars to Subgraphs: Uplifting Any {GNN} with Local Structure Awareness},
author={Lingxiao Zhao and Wei Jin and Leman Akoglu and Neil Shah},
booktitle={International Conference on Learning Representations},
year={2022},
url={https://openreview.net/forum?id=Mspk_WYKoEH}
}

@inproceedings{zhao2022a,
author = {Zhao, Lingxiao and H\"{a}rtel, Louis and Shah, Neil and Akoglu, Leman},
title = {A practical, progressively-expressive {GNN}},
year = {2022},
isbn = {9781713871088},
publisher = {Curran Associates Inc.},
address = {Red Hook, NY, USA},
booktitle = {Proceedings of the 36th International Conference on Neural Information Processing Systems},
articleno = {2472},
numpages = {15},
location = {New Orleans, LA, USA},
series = {NIPS '22}
}

@inproceedings{sarkar2010theoretical,
  author       = {Purnamrita Sarkar and
                  Deepayan Chakrabarti and
                  Andrew W. Moore},
  editor       = {Adam Tauman Kalai and
                  Mehryar Mohri},
  title        = {Theoretical Justification of Popular Link Prediction Heuristics},
  booktitle    = {{COLT} 2010 - The 23rd Conference on Learning Theory, Haifa, Israel,
                  June 27-29, 2010},
  pages        = {295--307},
  publisher    = {Omnipress},
  year         = {2010},
  bibsource    = {dblp computer science bibliography, https://dblp.org}
}

@inproceedings{saxe2014exact,
  author       = {Andrew M. Saxe and
                  James L. McClelland and
                  Surya Ganguli},
  editor       = {Yoshua Bengio and
                  Yann LeCun},
  title        = {Exact solutions to the nonlinear dynamics of learning in deep linear
                  neural networks},
  booktitle    = {2nd International Conference on Learning Representations, {ICLR} 2014,
                  Banff, AB, Canada, April 14-16, 2014, Conference Track Proceedings},
  year         = {2014},
  url          = {http://arxiv.org/abs/1312.6120},
  bibsource    = {dblp computer science bibliography, https://dblp.org}
}

@article{hoff2002latent,
  author = {Peter D Hoff and Adrian E Raftery and Mark S Handcock},
  title = {Latent Space Approaches to Social Network Analysis},
  journal = {Journal of the American Statistical Association},
  volume = {97},
  number = {460},
  pages = {1090--1098},
  year = {2002},
  publisher = {Taylor \& Francis},
  doi = {10.1198/016214502388618906},
  URL = {https://doi.org/10.1198/016214502388618906},
  eprint = {https://doi.org/10.1198/016214502388618906}
}
\bibliographystyle{icml2026}

\newpage
\appendix
\onecolumn

\section{Adjacency Reconstruction for Multiplicative Residual Connection}
\label{app:adjacency-recovery}

The node-adjacency encoding scheme \cite{yehudai_depth-width_2025} used in PENCIL is constructed so that the sampled subgraph connectivity can be recovered directly from the input tensor $\tilde{\mathbf{X}}$, eliminating the need to separately batch an adjacency structure.

Recall that each mini-batch contains $B$ sampled subgraphs. For sample $b$, let $N_b$ be the number of sampled \emph{context} nodes, and define $N_B=\max_{b\in[B]}N_b$ as the maximum number of context nodes in the batch after padding. Let $N_{\max}$ be the global sampling budget (maximum allowed context nodes per subgraph over the dataset). The tokenized input \(\tilde{\mathbf{X}} \in \mathbb{R}^{B \times (N_B+2) \times (2N_{\max}+2)}\) contains $N_B$ context tokens plus two \emph{task} tokens (for $v_{\text{src}}$ and $v_{\text{dst}}$). Each token row concatenates (i) a padded one-hot identifier over the $N_{\max}$ context-index slots, (ii) a padded adjacency-indicator row over the same $N_{\max}$ slots, and (iii) a 2-bit role flag (see Figure~\ref{fig:adjacency_row}). Importantly, the two task-token rows copy the identifier and adjacency parts of their corresponding endpoint context tokens, differing only in the role flag.

\paragraph{Reconstructing a context-sourced adjacency operator.}
We recover the subgraph adjacency operator by slicing out the identifier and adjacency blocks from $\tilde{\mathbf{X}}$. Using Python-style indexing, define
\begin{equation}
\tilde{\mathbf{X}}^{\text{id}} \;=\; \tilde{\mathbf{X}}_{:,:,\,0:N_B}, 
\qquad 
\tilde{\mathbf{X}}^{\text{adj}} \;=\; \tilde{\mathbf{X}}_{:,:,\,N_{\max}:N_{\max}+N_B}.
\end{equation}
Both slices have shape $B\times (N_B+2)\times N_B$. We slice only the first $N_B$ columns (rather than all $N_{\max}$) to restrict the operator to the active, non-padded context slots in the current batch.

Intuitively, $\tilde{\mathbf{X}}^{\text{adj}}$ provides the \emph{off-diagonal} neighborhood indicators: the entry $\tilde{\mathbf{X}}^{\text{adj}}_{b,i,j}=1$ indicates that, in sample $b$, token $i$ is connected to the $j$-th context token (under the subgraph indexing used to form $\tilde{\mathbf{X}}$). The slice $\tilde{\mathbf{X}}^{\text{id}}$ provides an \emph{identity link} to the corresponding context index, i.e., it plays the role of adding self-loops for context tokens, and (because task tokens copy the endpoint identifier) it also ensures that each task token can directly aggregate its endpoint's own context representation through the propagation branch.

We therefore define a \emph{context-sourced} adjacency operator (rows are all tokens; columns are context tokens only):
\begin{equation} \label{eq:adj-construction}
\tilde{\mathbf{A}}_{\text{src}} 
\;=\; 
\tilde{\mathbf{X}}^{\text{adj}} + \tilde{\mathbf{X}}^{\text{id}}
\;\in\;
\mathbb{R}^{B\times (N_B+2)\times N_B}.
\end{equation}
The name $\tilde{\mathbf{A}}_{\text{src}}$ emphasizes that only the $N_B$ context tokens act as \emph{message sources} in this operator, while \emph{all} $N_B+2$ tokens (including the two task tokens) can act as \emph{receivers}.

\paragraph{Extending to a square operator.}
For the multiplicative residual branch, it is convenient to work with a square $(N_B+2)\times (N_B+2)$ operator. Since the encoding does not allocate dedicated column slots for the two task tokens (task tokens are virtual readout tokens rather than sampled subgraph nodes), we append two all-zero columns:
\begin{equation}
\tilde{\mathbf{A}}
\;=\;
\big[\,\tilde{\mathbf{A}}_{\text{src}} \ \ \mathbf{0}\,\big],
\qquad
\mathbf{0}\in\mathbb{R}^{B\times (N_B+2)\times 2}.
\end{equation}
This yields $\tilde{\mathbf{A}}\in\mathbb{R}^{B\times (N_B+2)\times (N_B+2)}$. By construction, the task tokens have zero outgoing columns and therefore do not act as sources in the propagation step; they behave as \emph{receive-only (sink) tokens}, aggregating messages from the context subgraph. This design is intentional: the Transformer self-attention already enables dense interactions among all tokens within each block, while the propagation branch injects an explicit one-hop, structure-respecting aggregation from the sampled context graph into both context and task representations. Empirically, we adopt the row-normalized form $\mathbf{D}^{-1}\tilde{\mathbf{A}}$ to enhance training stability.

\section{Supplementary Materials for Theoretical Analysis} \label{app:proofs}

All theoretical results are restated first before the proof for the reader's convenience.

\subsection{Proof for Theorem \ref{thm:dist_perm_invariance}}\label{app:proof_dist_perm_invariance}

We first state the formal statement of Theorem \ref{thm:dist_perm_invariance}.

\begin{theorem}
    \label{thm:dist_perm_invariance_single_f}
    Let $\mathcal{V}=\{0,1,\dots,N-1\}$ and let $\mathbf{A}\in\{0,1\}^{N\times N}$ be the adjacency matrix of a graph on $\mathcal{V}$.
    Fix an ordered query pair $(u,v)\in \mathcal{V}\times \mathcal{V}$ with $u\neq v$.
    Define the endpoint-constrained permutation set
    \[
    \Gamma_{u,v} \;:=\; \{\rho\in S_N:\ \rho(u)=0,\ \rho(v)=1\}.
    \]
    Let $\rho \sim \mathrm{Unif}(\Gamma_{u,v})$ be a random permutation drawn uniformly from the endpoint-constrained permutation set and let $f$ be any deterministic measurable function.
    Define the randomized predictor
    \[
    S(\mathbf{A};u,v) \;:=\; f\!\left(\mathbf{P}_\rho\, \mathbf{A}\, \mathbf{P}_\rho^\top\right).
    \]
    Then for any permutation $\pi\in S_N$, which is the symmetric group on $N$ elements, letting
    \[
    \mathbf{A}' := \mathbf{P}_\pi\, \mathbf{A}\, \mathbf{P}_\pi^\top,\qquad u' := \pi(u),\qquad v' := \pi(v).
    \]
    Then, the distribution of $S(\mathbf{A};u,v)$ is identical to that of $S(\mathbf{A}';u',v')$:
    \[
    S(\mathbf{A};u,v)\ \overset{d}{=}\ S(\mathbf{A}';u',v').
    \]
\end{theorem}

The endpoint-constrained randomness $\rho$ models the stochastic reindexing of non-query nodes (Section~\ref{sec:preliminaries}). In contrast, $\pi \in S_N$ denotes an arbitrary node relabeling, reflecting the standard permutation symmetry that graph predictors should satisfy. Theorem~\ref{thm:dist_perm_invariance_single_f} establishes that $S$ is permutation-invariant in distribution under these relabelings. Next, we proceed with the proof.

\begin{proof}
    Fix any $\pi\in S_N$. Let
    \[
    \mathbf{A}' := \mathbf{P}_\pi \mathbf{A} \mathbf{P}_\pi^\top,\qquad u' := \pi(u),\qquad v' := \pi(v),
    \]
    and let $\rho' \sim \mathrm{Unif}(\Gamma_{u',v'})$. By definition,
    \[
    S(\mathbf{A}';u',v') = f\!\left(\mathbf{P}_{\rho'} \mathbf{A}' \mathbf{P}_{\rho'}^\top\right).
    \]
    To prove $S(\mathbf{A};u,v)\overset{d}{=}S(\mathbf{A}';u',v')$, it suffices to show that for every measurable set $B\subseteq\mathbb{R}$,
    \[
    \Pr\!\big(S(\mathbf{A};u,v)\in B\big)=\Pr\!\big(S(\mathbf{A}';u',v')\in B\big).
    \]
    Since $\rho'$ is uniform over the finite set $\Gamma_{u',v'}$, we can expand
    \begin{align}
    \Pr\!\big(S(\mathbf{A}';u',v')\in B\big)
    = \Pr\!\Big(f\!\left(\mathbf{P}_{\rho'} \mathbf{A}' \mathbf{P}_{\rho'}^\top\right)\in B\Big) = \frac{1}{|\Gamma_{u',v'}|}\sum_{\rho'\in\Gamma_{u',v'}}
    \mathbbm{1}\!\left\{ f\!\left(\mathbf{P}_{\rho'} \mathbf{A}' \mathbf{P}_{\rho'}^\top\right)\in B \right\}. \label{eq:sum_start}
    \end{align}
    Substituting $\mathbf{A}' = \mathbf{P}_\pi \mathbf{A} \mathbf{P}_\pi^\top$ yields
    \begin{align}
    \mathbf{P}_{\rho'} \mathbf{A}' \mathbf{P}_{\rho'}^\top
    = \mathbf{P}_{\rho'} (\mathbf{P}_\pi \mathbf{A} \mathbf{P}_\pi^\top) \mathbf{P}_{\rho'}^\top
    = (\mathbf{P}_{\rho'}\mathbf{P}_\pi)\, \mathbf{A}\, (\mathbf{P}_{\rho'}\mathbf{P}_\pi)^\top = \mathbf{P}_{\rho'\circ\pi}\, \mathbf{A}\, \mathbf{P}_{\rho'\circ\pi}^\top. \label{eq:expand}
    \end{align}

    Plugging \eqref{eq:expand} into \eqref{eq:sum_start} gives
    \begin{equation}
    \Pr\!\big(S(\mathbf{A}';u',v')\in B\big)
    = \frac{1}{|\Gamma_{u',v'}|}\sum_{\rho'\in\Gamma_{u',v'}}
    \mathbbm{1}\!\left\{ f\!\left(\mathbf{P}_{\rho'\circ\pi}\, \mathbf{A}\, \mathbf{P}_{\rho'\circ\pi}^\top\right)\in B \right\}. \label{eq:sum_mid}
    \end{equation}
    
    Let $\rho := \rho'\circ\pi$. We claim that the map \(
    \varphi_\pi:\Gamma_{u',v'}\to\Gamma_{u,v}\) such that \( \varphi_\pi(\rho')=\rho'\circ\pi\)
    is a bijection. Indeed, if $\rho'\in\Gamma_{u',v'}$, we have:
    \[
    (\rho'\circ\pi)(u)=\rho'(\pi(u))=\rho'(u')=0,\qquad
    (\rho'\circ\pi)(v)=\rho'(\pi(v))=\rho'(v')=1,
    \]
    so $\rho'\circ\pi\in\Gamma_{u,v}$. Conversely, if $\rho\in\Gamma_{u,v}$ then $\rho\circ\pi^{-1}\in\Gamma_{u',v'}$ and
    $\varphi_\pi(\rho\circ\pi^{-1})=\rho$. Hence, $\varphi_\pi$ is bijective, and in particular
    $|\Gamma_{u',v'}|=|\Gamma_{u,v}|$.
    
    Therefore, we can rewrite the sum in \eqref{eq:sum_mid} over $\rho\in\Gamma_{u,v}$ to obtain
    \begin{align}
    \Pr\!\big(S(\mathbf{A}';u',v')\in B\big)
    &= \frac{1}{|\Gamma_{u,v}|}\sum_{\rho\in\Gamma_{u,v}}
    \mathbbm{1}\!\left\{ f\!\left(\mathbf{P}_{\rho}\, \mathbf{A}\, \mathbf{P}_{\rho}^\top\right)\in B \right\} \nonumber\\
    &= \Pr\!\Big(f\!\left(\mathbf{P}_{\rho} \mathbf{A} \mathbf{P}_{\rho}^\top\right)\in B\Big) \nonumber\\
    &= \Pr\!\big(S(\mathbf{A};u,v)\in B\big), \label{eq:finish}
    \end{align}
    where the second equality uses $\rho\sim \mathrm{Unif}(\Gamma_{u,v})$.
    \eqref{eq:finish} holds for all measurable $B$, which concludes the proof.
\end{proof}

\subsection{Simulation of Non-Residual Message Passing by a Residual MPNN}

We introduce the following lemma to show that the residual MPNN (e.g. the MPNN obtained from PENCIL when self-attention is replaced by the identity map) can simulate a non-residual MPNN.

\begin{lemma}
\label{lem:residual-simulates-nonresidual}
Let
\[
    \mathbf X^{(k)}
    =
    \mathcal M_k^G\!\left(\mathbf X^{(k-1)}\right),
    \qquad
    k=1,\ldots,K,
\]
be a $K$-layer non-residual MPNN, where $\mathcal M_k^G:\mathbb R^{n\times d}\to\mathbb R^{n\times d}$ is any message-passing layer. In particular, its neighborhood aggregation may be sum, mean, max, or min. Then a $K$-layer residual MPNN of hidden width $D=(K+1)d$ can represent every iterate $\mathbf X^{(0)},\ldots,\mathbf X^{(K)}$ exactly.
\end{lemma}

\begin{proof}
Partition the hidden representation as
\[
    \mathbf H^{(k)}
    =
    [\mathbf C_0^{(k)}
     \|\mathbf C_1^{(k)}
     \|\cdots
     \|\mathbf C_K^{(k)}],
\]
and initialize
\[
    \mathbf H^{(0)}
    =
    [\mathbf X^{(0)}\|\mathbf 0\|\cdots\|\mathbf 0].
\]
At layer $k$, choose the message-passing branch to read block $k-1$,
apply $\mathcal M_k^G$, and write the result into block $k$:
\[
    \mathbf H^{(k)}
    =
    \mathbf H^{(k-1)}
    +
    [\mathbf 0\|\cdots\|
      \mathcal M_k^G(\mathbf C_{k-1}^{(k-1)})
      \|\cdots\|\mathbf 0].
\]

Since block $k$ is zero before layer $k$,
\[
    \mathbf C_k^{(k)}
    =
    \mathcal M_k^G(\mathbf C_{k-1}^{(k-1)})
    =
    \mathbf X^{(k)}.
\]
All previously populated blocks are preserved by the residual term.
Inductively,
\[
    \mathbf H^{(K)}
    =
    [\mathbf X^{(0)}
     \|\mathbf X^{(1)}
     \|\cdots
     \|\mathbf X^{(K)}].
\]
\end{proof}

It is worth noting that \(D=(K+1)d\) is only a sufficient width for the block-coordinate proof under the original PENCIL update; it is not an inherent requirement of the model. The extra width can be avoided by modifying the architecture so that the Transformer and message-passing paths can be controlled separately, for example through a learnable gate, or by applying message passing before a Transformer block that can be set to the identity. Lastly, \(\tilde{\mathbf A}\) is constructed to include self-loops in Eq.~\ref{eq:adj-construction}; however, one can opt to exclude the self-loops so that message passing is performed on the sampled subgraph only. For the sake of simplicity, we refer $\tilde{\mathbf A}$ to the loop-free context adjacency for the following theoretical results.

\subsection{Neural Bellman--Ford Networks} \label{app:nbfnet}

We first summarize the generalized Bellman--Ford algorithm \cite{baras_path_2010}, rewriting the notation for consistency with our presentation. For the formulation used in knowledge graph setting and its connection to NBFNet, see \cite{zhu_neural_bellman-ford_2021}.

\begin{definition}[Generalized Bellman--Ford Algorithm]
    \label{def:generalized-bellman-ford}
    Let $G=(\mathcal{V},\mathcal{E})$ be a (directed) graph and fix a source node $u$.  
    Let $(\mathcal{S},\oplus,\otimes)$ be a semiring with additive identity \circled{0} and multiplicative identity \circled{1}. Define the incoming edge set of $v$ as \(
    \mathcal{I}(v) := \{\, (x,v)\in\mathcal{E}\,\}.
    \)
    Let $\mathbf{h}_v^{(t)} \in \mathcal{S}$ denote the node representation of $v$ at iteration $t$, and let
    $\mathbf{w}(x,v)\in\mathcal{S}$ denote the edge representation of edge $(x,v)$.
    Here $\mathbbm{1}(\cdot)$ is an indicator that outputs \circled{1} if $v = u$ and \circled{0} otherwise.
    
    The generalized Bellman--Ford updates are
    \begin{align}
        \mathbf{h}_v^{(0)} &\leftarrow \mathbbm{1}(v = u), \label{eq:bellman-ford-init} \\
        \mathbf{h}_v^{(t)} &\leftarrow \left( \bigoplus_{(x,v) \in \mathcal{I}(v)} \mathbf{h}_x^{(t-1)} \otimes \mathbf{w}(x, v) \right) \oplus \mathbf{h}_v^{(0)}. \label{eq:bellman-ford-iter}
    \end{align}
\end{definition}

It has been shown that, with appropriate choices of the semiring operators $(\oplus,\otimes)$, the generalized Bellman--Ford recursion recovers a broad class of path-based heuristics for link prediction \cite{zhu_neural_bellman-ford_2021}. 
NBFNet departs from the strict semiring setting by parameterizing the update with learnable functions: the indicator initialization is replaced by a learnable source embedding, $\otimes$ is instantiated as a message function, and $\oplus$ is instantiated as a permutation-invariant aggregation function. 
This modification yields a source-conditioned MPNN, where representations $\mathbf{h}_v^{(t)}$ are computed conditioned on the chosen source node $u$ \cite{you2021identity}. 

\subsection{Proof of Proposition~\ref{prop:PENCIL-degenerates-nbfnet}} \label{app:proof_PENCIL-degenerates-nbfnet}

\textbf{Proposition~\ref{prop:PENCIL-degenerates-nbfnet}.} \textit{There exists a parameter setting of PENCIL under which its layerwise update reduces to a source-conditioned MPNN, with readout at the canonical destination token \(v_1\).}

\begin{proof}
    By canonicalization, the source is always token $v_0$. Let $d$ be the hidden width of the source-conditioned MPNN to be simulated, and set the PENCIL hidden width to $D=(K+1)d$. We can choose the input projection such that
    \[
        \mathbf{h}_v^{(0)}
        =
        \left[
            \mathbbm{1}(v=v_0)\mathbf{e}
            \,\middle\|\,
            \mathbf{0}
            \,\middle\|\,
            \cdots
            \,\middle\|\,
            \mathbf{0}
        \right]
        \in\mathbb{R}^{D},
    \]
    for a learnable $\mathbf{e}\in\mathbb{R}^d$, matching NBFNet's source-conditioned initialization. Setting $\mathbf{T}_k=\operatorname{Id}$ for all $k$ replaces the self-attention transformation by the identity map, and Eq.~\ref{eq:H-k} reduces to a sum-aggregation MPNN with a residual term. By choosing $\mathbf{P}_k$ as in Lemma~\ref{lem:residual-simulates-nonresidual}, the resulting residual MPNN exactly simulates the corresponding non-residual source-conditioned MPNN. Reading out the final coordinate block at the canonical destination token $v_1$ therefore yields a source-conditioned link score.
\end{proof}

\subsection{Proof of Corollary~\ref{cor:global-heuristics}} \label{app:proof_global-heuristics}

\textbf{Corollary~\ref{cor:global-heuristics}.} 
\textit{
    Under suitable parameter settings and operator choices, PENCIL can realize a broad class of classical path-based link prediction scores and graph algorithms, including Katz index, Personalized PageRank, SPD, widest path, and most reliable path.
}

\begin{proof}
    By Proposition~\ref{prop:PENCIL-degenerates-nbfnet} and Lemma~\ref{lem:residual-simulates-nonresidual}, PENCIL can simulate the corresponding non-residual source-conditioned MPNN. Instantiating its boundary condition, message operator, aggregation operator, and edge weights with those of the generalized Bellman--Ford algorithm therefore makes the PENCIL iterates coincide with the generalized Bellman--Ford iterates. \citet{zhu_neural_bellman-ford_2021} show that the listed heuristics are recovered under their respective semiring operators and edge-weight choices, which concludes the proof.
\end{proof}

\subsection{Formal statements of Remark~\ref{remark:dot-products}} \label{app:formal-theorems}

The following results are established in \cite{dong_pure_2024}; we restate them here with minor rephrasing to align with our notation and paper context.

\begin{theorem}{\cite{dong_pure_2024}} \label{thm:common-neighbor-estimation}
    Let $\mathbf{h}_v^{(0)} \in \mathbb{R}^d$ be initial node representations where each node vector is a zero-mean unit-norm vector in expectation. Under a single layer of sum-aggregation message passing, the inner product of any two node embeddings is an unbiased estimator of their common neighbor count:
    \begin{equation}
        \mathbb{E}\!\left[\mathbf{h}_u^{(1)} \cdot \mathbf{h}_v^{(1)}\right] = |\mathcal{N}(u) \cap \mathcal{N}(v)|.
    \end{equation}
\end{theorem}

\begin{theorem}{\cite{dong_pure_2024}} \label{thm:common-neighbor-estimation-general}
    Under the same conditions as Theorem~\ref{thm:common-neighbor-estimation}, given $p$ and $q$ iterations of sum-aggregation message passing, the expected inner product between
    the resulting node embeddings satisfies
    \begin{equation}
        \mathbb{E}\!\left[\mathbf{h}_u^{(p)} \cdot \mathbf{h}_v^{(q)}\right]
        \;=\;
        \sum_{k\in V} \, \big|\mathrm{walks}^{(p)}(k,u)\big| \, \big|\mathrm{walks}^{(q)}(k,v)\big|,
    \end{equation}
    where $\big|\mathrm{walks}^{(\ell)}(a,b)\big|$ denotes the number of length-$\ell$ walks between nodes $a$ and $b$.
\end{theorem}

\subsection{Proof of Proposition~\ref{prop:local-heuristics}} \label{app:proof_local-heuristics}

\textbf{Proposition~\ref{prop:local-heuristics}.} \textit{
    There exists a parameter setting of PENCIL under which it can estimate heuristics stated in Theorem~\ref{thm:common-neighbor-estimation} and Theorem~\ref{thm:common-neighbor-estimation-general}.
}

\begin{proof}
    Set the PENCIL hidden width to
    $D=(K+1)d$, where \(K \geq \max\{p, q\}\). We can choose the input projection such that, for each
    sampled token $v_i$, $i\in[N]$,
    \[
        \mathbf{h}_{v_i}^{(0)}
        =
        \left[
            \mathbf{e}_i
            \,\middle\|\,
            \mathbf{0}
            \,\middle\|\,
            \cdots
            \,\middle\|\,
            \mathbf{0}
        \right]
        \in\mathbb{R}^{D},
    \]
    where the vectors $\mathbf{e}_i\in\mathbb{R}^d$ are sampled independently from a normalized Rademacher distribution and then fixed. Setting $\mathbf{T}_k=\operatorname{Id}$ for all $k$ replaces the self-attention transformation by the identity map, and Eq.~\ref{eq:H-k} reduces to a sum-aggregation MPNN with a residual term. By choosing $\mathbf{P}_k$ as in Lemma~\ref{lem:residual-simulates-nonresidual}, the resulting residual MPNN exactly simulates the corresponding non-residual sum-aggregation MPNN and therefore meets the conditions of Theorem~\ref{thm:common-neighbor-estimation} and Theorem~\ref{thm:common-neighbor-estimation-general}. Assume the final link scorer function is a non-linear MLP. Hence, by the universal approximation theorem, a nonlinear MLP jointly applied to the endpoint representations can approximate their inner product, and therefore the stated heuristics.
\end{proof}

\subsection{Proof of Proposition~\ref{prop:monotonicity-of-welch-bound}} \label{app:proof_monotonicity-of-welch-bound-proposition}

\textbf{Proposition~\ref{prop:monotonicity-of-welch-bound}.} \textit{
    For a fixed dimension $d \ge 2$, the Welch bound $W(N,d)$ is strictly increasing as a function of the set size $N$ for all $N > d$.
}

\begin{proof}
    Fix \(d\ge 2\) and let \(M>N>d\). Since \(\sqrt{\cdot}\) is strictly increasing on \((0,\infty)\), it suffices to show
    \[
    \frac{M-d}{M-1}>\frac{N-d}{N-1}.
    \]
    Because \((M-1)(N-1)>0\), cross-multiplying yields
    \[
    \frac{M-d}{M-1}>\frac{N-d}{N-1}
    \;\Longleftrightarrow\;
    (M-d)(N-1)>(N-d)(M-1).
    \]
    Rearranging,
    \[
    (M-d)(N-1)-(N-d)(M-1)=(M-N)(d-1)>0,
    \]
    since \(M>N\) and \(d\ge 2\). Therefore \(W(M,d)>W(N,d)\), and \(W(N,d)\) is strictly increasing in \(N\) for all \(N>d\).
\end{proof}

\subsection{Expressive Power of PENCIL} \label{app:expressive-power-of-PENCIL}

\begin{table}[t]
    \centering
    \caption{Model formulations expressed within the $k_{\phi}$-$k_{\rho}$-$m$ framework. A `/' indicates that the corresponding component is not included in the model. Table is extracted from \cite{lachi2025bridgingtheorypracticelink}, where PENCIL is our contribution.}
    \label{tab:expressiveness-comparison}
    \resizebox{\textwidth}{!}{
    \renewcommand{\arraystretch}{1.5}
    \begin{tabular}{lcccccccc}
    \toprule
    Model & COMB & $g$ & $k_{\phi}$ & AGG & $\psi$ & $k_{\rho}$ & $m$ & $h$ \\
    \midrule
    Pure GNN & / & $\odot$ & 1-WL & / & / & / & / & / \\
    NCN & $\|$ & $\odot$ & 1-WL & $\sum$ & $\rho(i; G, \mathbf{X}^0)$ & 1-WL & 1 & $\mathbf{X}^0$ \\
    ELPH & $\|$ & $\|$ & 1-WL & $\sum$ & $\rho(i; G, \mathbf{X}^1) \cdot \prod_{r=1}^{m} \prod_{d=1}^{m} \mathbbm{1}_{dr}(i)$ & 1-WL & m & $\mathbf{x}_i^1 = 1$ \\
    Neo-GNN & $\|$ & $\|$ & 1-WL & $\sum$ & \begin{tabular}[c]{@{}c@{}}$b \cdot \rho(i; G, \mathbf{X}^0)$ with\\ $b = \sum_{r=1}^{m} \sum_{d=1}^{m} (\mathbf{A}^r)_{uv} \cdot (\mathbf{A}^d)_{uv}$\end{tabular} & 1-WL & m & $\mathbf{X}^0$ \\
    SEAL & / & / & / & $\sum$ & $\rho(i; G, \mathbf{X}^D)$ & $1-|N^m(u,v)|\text{-WL}$ & m & $\mathbf{x}_i^D = \mathbf{x}_i^0 \| \min_{u,v}(\delta(i, u), \delta(i, v)) + 1$ \\
    \midrule
    PENCIL & / & / & / & $\sum$ & \begin{tabular}[c]{@{}c@{}}$\psi_{\text{end}}\!\big(\rho(i;G,\tilde{\mathbf{X}}),u,v\big)$ with\\ $\psi_{\text{end}}(z_i,u,v)=\big[\mathbbm{1}(i=u)\,z_i \ \| \ \mathbbm{1}(i=v)\,z_i\big]$\end{tabular} & $1-|N^m(u,v)|\text{-WL}$ & m & $\tilde{\mathbf{x}}_i=\mathbf{r}_i$ \\
    \bottomrule
    \end{tabular}%
    }
\end{table}

We first briefly review the \(k_\phi\text{-}k_\rho\text{-}m\) framework proposed in \cite{lachi2025bridgingtheorypracticelink}, which is the framework where our expressive power analysis is conducted. Then, we align PENCIL with the framework.

\begin{definition}[$k_\phi$-$k_\rho$-$m$ framework]
    Let $G=(\mathcal{V},\mathcal{E})$ be a graph with initial node-feature matrix $\mathbf{X}^0$. For brevity, we suppress the explicit dependence of all functions on $G$ unless needed.

    A link-representation MPNN $M$ is said to belong to the $k_\phi$-$k_\rho$-$m$ framework if its output for a node pair $(u,v)$ can be written as
    \begin{gather}
    F\big((u,v),\mathbf{X}^0\big)
    = \mathrm{COMB}\Big(
        g\big(\phi(u,\mathbf{X}^0),\,\phi(v,\mathbf{X}^0)\big),\;
        \mathrm{AGG}\big(\{\, f(i,u,v,\mathbf{X}^0)\mid i \in \bigcup_{j=0}^{m} \mathcal{N}^{j}(u,v)\,\}\big)
    \Big), \\
    f(i,u,v,\mathbf{X}^0)
    = \psi\big(\rho(i,\,h(u,v,\mathbf{X}^0)),\,u,\,v\big),
    \end{gather}
    where:

    \begin{itemize}
        \item $\phi$ and $\rho$ are MPNNs with expressive power $k_\phi$ and $k_\rho$, respectively.
        \item For $m\ge 1$, $\mathcal{N}^m(v)$ denotes the set of nodes at exactly $m$ hops from $v$. For a pair $(u,v)$, define the joint $m$-hop neighborhood $\mathcal{N}^m(u,v) := \mathcal{N}^m(u)\cup \mathcal{N}^m(v)$.
        \item $h(u,v,\mathbf{X}^0)\in\mathbb{R}^{|\mathcal{V}|\times d}$ is a derived node-feature matrix computed from $\mathbf{X}^0$, optionally augmented with pair-specific information for $(u,v)$.
        \item $\psi$ applies a pair-dependent rescaling to message-passing representations, potentially using information specific to $(u,v)$.
        \item $g$ aggregates the node representations of endpoints $\phi(u,\mathbf{X}^0)$ and $\phi(v,\mathbf{X}^0)$ into an endpoints' representation.
        \item $\mathrm{AGG}$ is a permutation-invariant aggregation operator over the set of node representations in the selected neighborhood, and $\mathrm{COMB}$ merges the endpoints' representation and neighborhood representations.
    \end{itemize}
    
\end{definition}

Table~\ref{tab:expressiveness-comparison} aligns existing link predictors---ELPH \cite{chamberlain2023graph}, Neo-GNN \cite{yun_neo-gnns_2021}, NCN \cite{wang2024neural}, and SEAL \cite{zhang_link_2018}---with the $k_\phi$-$k_\rho$-$m$ framework \cite{lachi2025bridgingtheorypracticelink}. Our primary comparison is between PENCIL and SEAL; the remaining models are included for context (see \cite{lachi2025bridgingtheorypracticelink} for details).

PENCIL and SEAL fit the framework in a similar \emph{link-centric} manner: neither model uses a separate endpoint encoder $\phi$ nor an explicit endpoint combiner $(g,\mathrm{COMB})$. Instead, both compute node representations on the extracted enclosing subgraph and then aggregate them into a link representation. Consequently, the relevant components for distinguishing the two models are $\mathrm{AGG}$, $\psi$, $k_\rho$, $m$, and $h$.

Both methods operate on the $m$-hop enclosing neighborhood of $(u,v)$, so they share the same $m$ in Table~\ref{tab:expressiveness-comparison}. SEAL constructs a distance-aware feature matrix $\mathbf{X}^D$ using DRNL labels \cite{zhang_link_2018} and applies an MPNN $\rho$ on this subgraph. In the framework, this corresponds to $h(u,v,\mathbf{X}^0)=\mathbf{X}^D$ and a pair-independent identity mapping $\psi(z_i,u,v)=z_i$, so the neighborhood-level representation reduces to aggregation over all nodes in the $m$-hop enclosing neighborhood:
$\sum_{i \in \bigcup_{j=0}^m \mathcal{N}^j(u,v)} \rho(i;G,\mathbf{X}^D)$.

In contrast, PENCIL uses an index-dependent input embedding $\tilde{\mathbf{X}}$ for the sampled subgraph nodes (Section~\ref{sec:plain-transformers-as-link-predictors}), which in Table~\ref{tab:expressiveness-comparison} is captured by $h(u,v,\mathbf{X}^0)=\tilde{\mathbf{X}}$ with $\tilde{\mathbf{x}}_i=\mathbf{r}_i$, where $\mathbf{r}_i$ is a random vector induced by the input projection matrix. The final link representation is formed by extracting and concatenating the endpoint-specific components, which can be written as an $\mathrm{AGG}=\sum$ over a gated construction
$\psi_{\mathrm{end}}(z_i,u,v) = [\mathbbm{1}(i=u)\,z_i \ \| \ \mathbbm{1}(i=v)\,z_i]$.
This choice of $\psi$ is a convenient readout that isolates the endpoint representations. Next, we restate Theorem~\ref{thm:expressivity-of-PENCIL} and provide its proof.

\textbf{Theorem~\ref{thm:expressivity-of-PENCIL}.} \textit{
    Under the \(k_\phi\text{-}k_\rho\text{-}m\) framework, for suitable parameter settings, PENCIL with LRP is not less expressive than SEAL under the same sampling constraint.}
\begin{proof}
    We first note that both models are \emph{local} around the queried edge: for a fixed integer \(m\), the prediction for \((u,v)\) is a function only of the \(m\)-hop enclosing neighborhood \(N^m(u,v)\). Hence, the two models share the same locality parameter \(m\) in the \(k_\phi\text{-}k_\rho\text{-}m\) framework.
    
    Next, consider PENCIL under a parameter setting with \(\mathbf{T}_k=\mathrm{Id}\) for all \(k\) and an input projection such that
    \(\mathbf{h}_v^{(0)}=\mathbbm{1}(v=v_i)\mathbf{e}_i\) for learnable vectors \(\mathbf{e}_i\in\mathbb{R}^d\). Under this setting, PENCIL reduces to a 1-WL MPNN operating on the sampled subgraph (see Eq.~\ref{eq:H-k}). Define the endpoint-constrained permutation set over the sampled subgraph as
    \[
    \Gamma_{u,v} \;:=\; \{\pi\in S_{|N^m(u,v)|}:\ \pi(u)=0,\ \pi(v)=1\},
    \]
    where \(S_N\) is the symmetric group on \(N\) elements. With LRP, we aggregate the MPNN outputs over the \(|\Gamma_{u,v}|\) endpoint-preserving relabelings; by \cite{zhou_relational_2023}, this yields \(1\text{-}|N^m(u,v)|\text{-WL}\) power, i.e., the same \(k_\rho\) as SEAL on \(N^m(u,v)\) \cite{lachi2025bridgingtheorypracticelink}. Therefore, since two models have the same \(m\) and \(k_\rho\), the theorem follows according to Theorem~3.2 in \cite{lachi2025bridgingtheorypracticelink}.
\end{proof}

Under the expressivity ordering in \cite{lachi2025bridgingtheorypracticelink}, the above \(1\text{-}|N^m(u,v)|\text{-WL}\) guarantee is strictly stronger than the expressivity attributed to several heuristics-based predictors, including ELPH \cite{chamberlain2023graph}, NCN \cite{wang2024neural}, and Neo-GNN \cite{yun_neo-gnns_2021}. Importantly, this result should \emph{not} be interpreted as an upper bound on PENCIL: it only establishes that PENCIL can \emph{degenerate} to \(1\text{-}|N^m(u,v)|\text{-WL}\) expressivity under a particular parameter setting. Characterizing the expressive power of the full attention-enabled PENCIL remains an interesting direction for future work.

Complementarily, \cite{liang_can_2025} argues that SEAL's Double Radius Node Labeling (DRNL) may not reliably expose common-neighbor signals in practice, since subsequent message passing and aggregation can wash out this information. In contrast, PENCIL's input projection injects token-specific vectors that make such neighborhood identity information directly accessible, as discussed and empirically validated in previous sections. Finally, we do not equip PENCIL with LRP in our experiments; empirically, we find that using a single labeled subgraph already suffices, as shown by experiments covered in Section~\ref{sec:experimental-results} and Appendix~\ref{app:multiple_labeled_subgraphs}.

\begin{figure}[t]
    \centering
    \includegraphics[width=\linewidth]{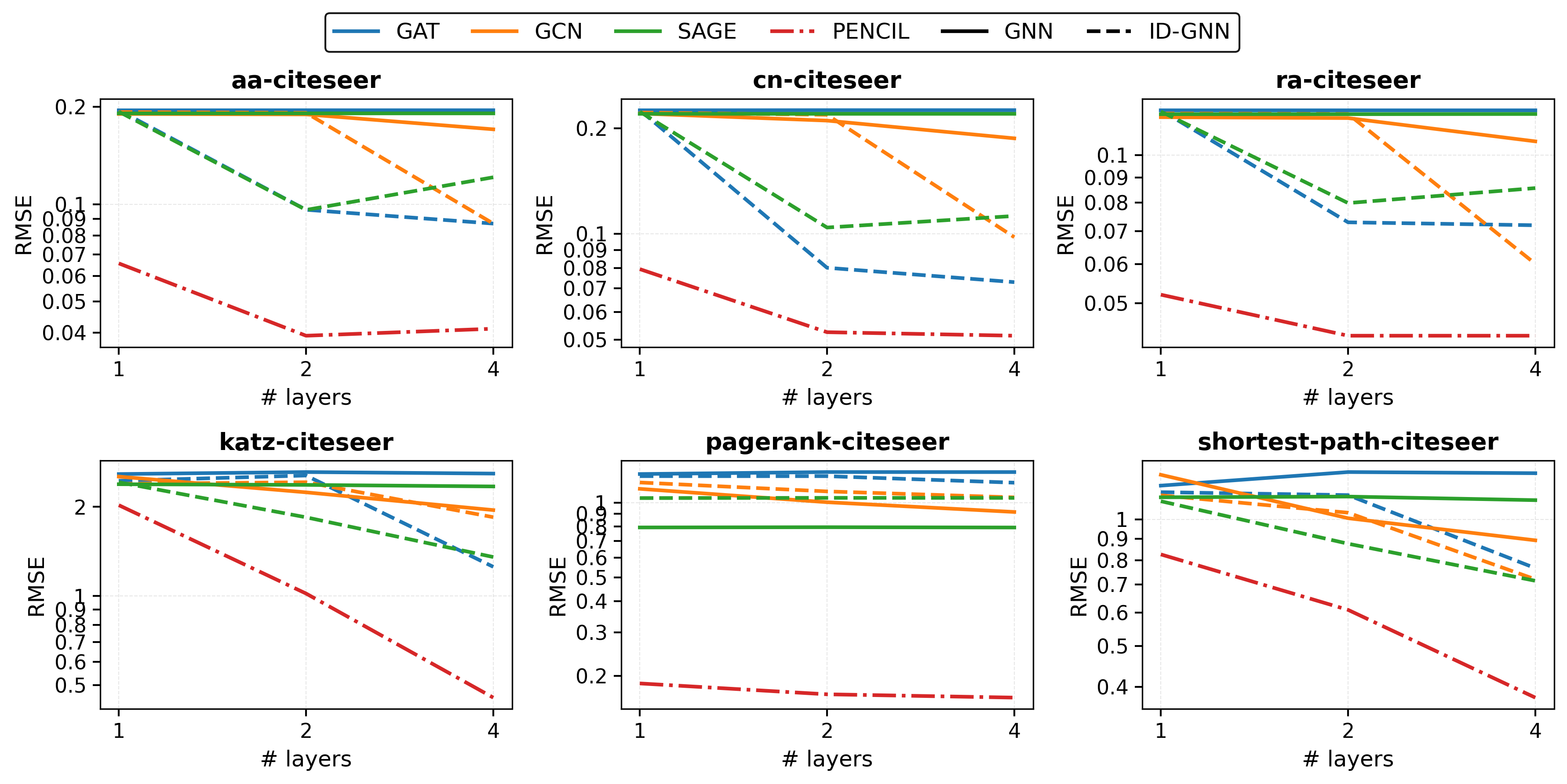}
    \caption{RMSE of PENCIL and other GNNs for estimating pairwise heuristics on the \texttt{citeseer} dataset.}
    \label{fig:citeseer_heuristic_rmse}
\end{figure}

\section{Experiments}
\label{app:experiments}

This section presents supplementary information regarding the experiments in Section~\ref{sec:experimental-results}, as well as further investigations into input projection matrix initialization and computational complexity. 

Regarding practical implementation, we employ the ShaDowKHop sampler \cite{zeng2021decoupling}, utilizing the implementation provided by GraphGPT \cite{zhao_graphgpt_2025}. This sampler requires two hyperparameters: depth $d$ and the number of neighbors per node $n$. Where necessary to distinguish configurations, we append the suffix $(d, n)$ to dataset names to denote the specific parameters used. For the Transformer layer, we employ the BERT implementation \cite{devlin-etal-2019-bert} provided by Hugging Face \cite{wolf-etal-2020-transformers}. All codes are written in Pytorch \cite{paszke2019pytorch} and Pytorch Geometric (PyG) \cite{fey2019pyg,fey2025pyg}.

\subsection{Pairwise Heuristic Estimation}
\label{app:pairwise-heuristic-estimation}

In this section, we describe the pairwise heuristics considered in our experiments, along with the normalization procedure applied to the raw heuristic scores so they can be used as regression targets, since many exhibit heavy-tailed distributions. A unified three-stage preprocessing protocol is used across all heuristics: (i) logarithmic transformation to reduce skewness; (ii) range mapping (normalization or standardization) to align scales; and (iii) robust clipping to prevent rare outliers from dominating the loss gradient. We mirror the link prediction pipeline for the heuristic regression task, including the sampling of one negative edge per positive training instance. However, we replace the binary classification target with a normalized heuristic score for regression. All experiments utilize the standard link prediction splits defined in \cite{li_evaluating_2023}. We compare PENCIL against standard GNN baselines---GCN, GAT, and GraphSAGE \cite{kipf2017semisupervised,veličković2018graph,hamilton2017inductive}---as well as identity-aware variants that augment message passing with explicit markers for the query nodes (e.g., source/target indicators) \cite{you2021identity,zhu_neural_bellman-ford_2021,zhang_labeling_2022}. These approaches share the common principle of injecting node identity information to facilitate reasoning about a specific pair, and have been shown to scale to large relational settings \cite{robinson2024relbench,yuan_contextgnn_2025}. Since we do not use node features, both GNNs and PENCIL rely solely on graph structure: GNNs take all-ones node inputs and learn representations through message passing, whereas PENCIL uses the structural encoding from Section~\ref{sec:preliminaries} and learns representations via self-attention over the encoded sequence. The results are shown in Figure~\ref{fig:cora_heuristic_rmse} and \ref{fig:citeseer_heuristic_rmse}. The hyperparameter configurations for GNNs and PENCIL in the pairwise heuristic estimation experiments are summarized in Table~\ref{tab:hyperparameters}. The hidden size of PENCIL was constrained to yield comparable parameter counts across all models. Models were trained for 200 epochs for CN, AA, RA, and PageRank, while the Katz index and SPD experiments required 700 epochs.

\begin{table}[t]
    \centering
    \caption{Hyperparameter configurations of GNNs and PENCIL for the pairwise heuristic estimation experiments. N/A means the hyperparameter is not applicable.}
    \label{tab:hyperparameters}
    \begin{tabular}{@{}lcc@{}}
    \toprule
    Hyperparameter & GNNs & PENCIL \\
    \midrule
    Sampling Configuration & (2, 20) & (2, 20) \\
    Hidden Size & 512 & 256 \\
    Intermediate Size & N/A & 512 \\
    \# Attention Heads & N/A & 4 \\
    Effective Batch Size & 2048 & 2048 \\
    Learning Rate & 2.0E-04 & 2.0E-04 \\
    Weight Decay & 0.01 & 0.01 \\
    \bottomrule
    \end{tabular}
\end{table}

\paragraph{Local Overlap Heuristics.}
Three local overlap scores based on the shared one-hop neighborhood of a candidate pair \((u,v)\) are considered:
\begin{align}
\mathrm{CN}(u,v) &:= \big|\mathcal{N}(u)\cap \mathcal{N}(v)\big|, \\
\mathrm{AA}(u,v) &:= \sum_{w\in \mathcal{N}(u)\cap \mathcal{N}(v)} \frac{1}{\log(\deg(w))}, \\
\mathrm{RA}(u,v) &:= \sum_{w\in \mathcal{N}(u)\cap \mathcal{N}(v)} \frac{1}{\deg(w)}.
\end{align}
These heuristics quantify structural overlap between \(u\) and \(v\): Common Neighbors (CN) \cite{newman_clustering_2001} counts the number of common neighbors, whereas Adamic--Adar (AA) \cite{adamic_friends_2003} and Resource Allocation (RA) \cite{zhou_predicting_2009} discount common neighbors \(w\) with large degree, thereby emphasizing rarer shared neighbors. Here, \(\mathcal{N}(u)\) denotes the set of one-hop neighbors of node \(u\), and \(\deg(w):=|\mathcal{N}(w)|\) denotes the degree of node \(w\). The summations range over the shared neighbors \(w\in \mathcal{N}(u)\cap \mathcal{N}(v)\).
Since these scores are non-negative and typically right-skewed, a log transformation \(s \mapsto \log(1+s)\) is applied, where \(s\in\{\mathrm{CN}(u,v),\mathrm{AA}(u,v),\mathrm{RA}(u,v)\}\). This is followed by min--max scaling that maps \([0, P_{99.99}]\) to \([0,1]\), where \(P_{99.99}\) denotes the \(99.99^{\text{th}}\) percentile of the log-transformed scores. Values exceeding this threshold are clipped to \(1\).

\paragraph{Katz Index.}
The Katz score for a pair \((u,v)\) is defined as a damped sum over walks of all lengths \cite{katz_new_1953}:
\begin{equation}
\mathrm{Katz}_{\beta}(u,v) \;:=\; \sum_{\ell=1}^{\infty} \beta^{\ell}\,(\mathbf{A}^{\ell})_{uv}.
\end{equation}
This global similarity measure aggregates multi-hop connectivity between \(u\) and \(v\) while exponentially down-weighting longer walks. Here, \(\mathbf{A}\in\{0,1\}^{|V|\times |V|}\) is the adjacency matrix, \((\mathbf{A}^{\ell})_{uv}\) is the number of length-\(\ell\) walks from \(u\) to \(v\), and \(\beta\in(0,1)\) is a damping factor (we use \(\beta=0.005\)) that emphasizes shorter walks. Katz scores are transformed as \(\log(k+\epsilon)\), where \(k=\mathrm{Katz}_{\beta}(u,v)\) and \(\epsilon=10^{-20}\), and then standardized to zero mean and unit variance. For robustness, the standardized scores are clipped in the upper tail at the \(99^{\text{th}}\) percentile.

\paragraph{Shortest Path Distance.}
The shortest-path distance between nodes \(u\) and \(v\) is defined as
\begin{equation}
\mathrm{SPD}(u,v) \;:=\; d(u,v) \;=\; \min_{\pi:u\rightsquigarrow v} |\pi|,
\end{equation}
where \(\pi\) ranges over all paths from \(u\) to \(v\) and \(|\pi|\) denotes the number of edges along the path. This heuristic measures how many steps are required to connect \(u\) and \(v\) via the shortest route; if \(u\) and \(v\) are disconnected, then \(d(u,v)=\infty\). Distances are normalized via the log transform \(d \mapsto \log(1+d)\). For disconnected pairs, a distinct penalty value \(\log(1+1.5\,d_{\max})\) is assigned, where \(d_{\max}\) is the maximum diameter observed among the connected components of the graph. No further percentile clipping is applied to SPD.

\paragraph{PageRank.}
Let \(\mathbf{p}\in\mathbb{R}^{|\mathcal{V}|}\), where \(\mathcal{V}\) is the node set, denote the PageRank stationary distribution \cite{brin_anatomy_1998}, defined as the solution to
\begin{equation}
\mathbf{p} \;=\; \alpha\,\mathbf{P}^{\top}\mathbf{p} \;+\; (1-\alpha)\,\mathbf{v},
\end{equation}
where \(\mathbf{P}\) is a random-walk transition matrix (e.g., \(\mathbf{P}=\mathbf{D}^{-1}\mathbf{A}\) with degree matrix \(\mathbf{D}\)), \(\alpha\in(0,1)\) is the damping factor, and \(\mathbf{v}\) is the teleportation distribution. PageRank quantifies node importance under this random-surfer model, and \(p_u\) denotes the stationary probability assigned to node \(u\). In our experiments, we use a uniform teleportation distribution and a damping factor of \(\alpha=0.85\). A symmetric pairwise score for \((u,v)\) is then given by their product
\(
\mathrm{PR}(u,v) \;:=\; p_u\,p_v.
\)
These pairwise scores can also be heavy-tailed, so the normalization protocol is applied similar to that for the Katz index: log-transformation \(\log(\mathrm{PR}(u,v)+\epsilon)\) with \(\epsilon=10^{-20}\), followed by z-score standardization, with upper-tail clipping at the \(99^{\text{th}}\) percentile.

\begin{table}[tbp]
    \centering
    \caption{Hyperparameter configurations of PENCIL for the original benchmark settings.}
    \label{tab:original_configs}
    \begin{tabular}{@{}lcccccc@{}}
    \toprule
    Hyperparameter         & \texttt{cora}    & \texttt{citeseer} & \texttt{pubmed}  & \texttt{ogbl-collab} & \texttt{ogbl-ppa} & \texttt{ogbl-citation2} \\ \midrule
    Sampling Configuration & (2, 20) & (2, 20)  & (2, 16) & (1, 75)     & (1, 150) & (1, 40)        \\
    Hidden Size            & 512     & 1024     & 512     & 512         & 512      & 512            \\
    Intermediate Size      & 2048    & 2048     & 2048    & 2048        & 2048     & 2048           \\
    \# Layers              & 8       & 2        & 8       & 8           & 4        & 3              \\
    \# Attention Heads     & 8       & 8        & 8       & 8           & 8        & 8              \\
    Effective Batch Size   & 2048    & 2048     & 4096    & 8192        & 16384    & 65536          \\
    Learning Rate          & 1.0E-04 & 1.0E-04  & 1.0E-04 & 1.0E-04     & 1.0E-04  & 1.0E-04        \\
    Weight Decay           & 0.01    & 0.01     & 0.01    & 0.01        & 0.01     & 0.01           \\
    \# Epochs              & 150     & 150      & 80      & 10          & 15       & 0.5            \\ \bottomrule
    \end{tabular}
\end{table}

\begin{table}[tbp]
    \centering
    \caption{Hyperparameter configurations of PENCIL for the HeaRT benchmark settings.}
    \label{tab:heart-config}
    \resizebox{\textwidth}{!}{%
    \begin{tabular}{@{}lccccccc@{}}
    \toprule
    Hyperparameter         & \texttt{cora}    & \texttt{citeseer} & \texttt{pubmed}  & \texttt{ogbl-collab} & \texttt{ogbl-ppa} & \texttt{ogbl-citation2} & \texttt{ogbl-ddi} \\ \midrule
    Sampling Configuration & (2, 20) & (2, 20)  & (2, 20) & (1, 75)     & (1, 150) & (1, 40)        & (1, 350) \\
    Hidden Size            & 512     & 1024     & 512     & 512         & 512      & 512            & 512      \\
    Intermediate Size      & 2048    & 2048     & 2048    & 2048        & 2048     & 2048           & 2048     \\
    \# Layers              & 4       & 2        & 4       & 8           & 4        & 3              & 8        \\
    \# Attention Heads     & 8       & 8        & 8       & 8           & 8        & 8              & 8        \\
    Effective Batch Size   & 2048    & 2048     & 2048    & 8192        & 16384    & 65536          & 4096     \\
    Learning Rate          & 1.0E-04 & 1.0E-04  & 1.0E-04 & 1.0E-04     & 1.0E-04  & 1.0E-04        & 1.00E-04 \\
    Weight Decay           & 0.01    & 0.01     & 0.01    & 0.01        & 0.01     & 0.01           & 0.01     \\
    \# Epochs              & 150     & 300      & 250     & 10          & 15       & 0.5            & 8        \\ \bottomrule
    \end{tabular}%
    }
\end{table}

\begin{table*}[t]
    \centering
    \caption{Results on the original benchmark datasets. 
    N/A means results are not available. Colored results indicate the \textcolor{firstbest}{\(1^{\text{st}}\)}, \textcolor{secondbest}{\(2^{\text{nd}}\)}, and \textcolor{thirdbest}{\(3^{\text{rd}}\)}-best performance in the corresponding metric. Category markers: 
    \raisebox{0.75ex}{\catdot{catgreen}} Pure Heuristic, 
    \raisebox{0.75ex}{\catdot{catblue}} Heuristic-informed, \raisebox{0.75ex}{\catdot{catorange}} ID-based, and \, \raisebox{0.75ex}{\catdot{catred}} Heuristic-agnostic + ID-free.
    }

    \label{tab:complete-original-setting}
    \small
    \resizebox{\textwidth}{!}{%
    \begin{tabular}{@{}clcccccc@{}}
    \toprule
    {\multirow{2}{*}{Type}} & \multirow{2}{*}{Model} & \texttt{cora}                                & \texttt{citeseer}                            & \texttt{pubmed}                              & \texttt{ogbl-collab}                         & \texttt{ogbl-ppa}                            & \texttt{ogbl-citation2}                      \\ \cmidrule(l){3-8} 
             &                               & MRR                                 & MRR                                 & MRR                                 & H@50                                & H@100                               & MRR                                 \\ \midrule
    \multirow{3}{*}{\raisebox{0.75ex}{\catdot{catgreen}}} 
             & CN                         & 20.99 ± 0.00                        & 28.34 ± 0.00                        & 14.02 ± 0.00                        & 61.37 ± 0.00                        & 27.65 ± 0.00                        & 74.30 ± 0.00                        \\
             & AA                         & 31.87 ± 0.00                        & 29.37 ± 0.00                        & 16.66 ± 0.00                        & 64.17 ± 0.00                        & 32.45 ± 0.00                        & 75.96 ± 0.00                        \\
             & RA                         & 30.79 ± 0.00                        & 27.61 ± 0.00                        & 15.63 ± 0.00                        & 63.81 ± 0.00                        & 49.33 ± 0.00                        & 76.04 ± 0.00                        \\ \midrule
    \multirow{6}{*}{\raisebox{0.75ex}{\catdot{catblue}}} 
             & SEAL                       & 26.69 ± 5.89                        & 39.36 ± 4.99                        & 38.06 ± 5.18                        & 63.37 ± 0.69                        & 48.80 ± 5.61                        & 86.93 ± 0.43                        \\
             & Neo-GNN                    & 22.65 ± 2.60                        & 53.97 ± 5.88                        & 31.45 ± 3.17                        & 66.13 ± 0.61                        & 48.45 ± 1.01                        & 83.54 ± 0.32                        \\
             & BUDDY                      & 26.40 ± 4.40                        & {\color{thirdbest} 59.48 ± 8.96}    & 23.98 ± 5.11                        & 64.59 ± 0.46                        & 47.33 ± 1.96                        & 87.86 ± 0.18                        \\
             & NCN                        & 32.93 ± 3.80                        & 54.97 ± 6.03                        & 35.65 ± 4.60                        & 63.86 ± 0.51                        & 62.63 ± 1.15                        & 89.27 ± 0.05                        \\
             & NCNC                       & 29.01 ± 3.83                        & {\color{secondbest} 64.03 ± 3.67}   & 25.70 ± 4.48                        & 65.97 ± 1.03                        & 62.61 ± 0.76                        & {\color{secondbest} 89.82 ± 0.43}   \\
             & LPFormer                   & {\color{secondbest} 39.42 ± 5.78}   & {\color{firstbest} 65.42 ± 4.65}    & {\color{secondbest} 40.17 ± 1.92}   & {\color{firstbest} 68.14 ± 0.51}    & 63.32 ± 0.63                        & {\color{thirdbest} 89.81 ± 0.13}    \\ \midrule
    \multirow{2}{*}{\raisebox{0.75ex}{\catdot{catorange}}} 
             & MPLP+                      & N/A                                 & N/A                                 & N/A                                 & {\color{secondbest} 66.99 ± 0.40}   & 65.24 ± 1.50                        & {\color{firstbest} 90.72 ± 0.12}    \\
             & Refined-GAE                & N/A                                 & N/A                                 & N/A                                 & 66.11 ± 0.35                        & {\color{secondbest} 78.41 ± 0.83}   & 88.74 ± 0.06                        \\ \midrule
    \multirow{5}{*}{\raisebox{0.75ex}{\catdot{catred}}}
             & GCN                        & 32.50 ± 6.87                        & 50.01 ± 6.04                        & 19.94 ± 4.24                        & 54.96 ± 3.18                        & 29.57 ± 2.90                        & 84.85 ± 0.07                        \\
             & SAGE                       & {\color{thirdbest} 37.83 ± 7.75}    & 47.84 ± 6.39                        & 22.74 ± 5.47                        & 59.44 ± 1.37                        & 41.02 ± 1.94                        & 83.06 ± 0.09                        \\
             & NBFNet                     & 37.69 ± 3.97                        & 38.17 ± 3.06                        & {\color{firstbest} 44.73 ± 2.12}    & N/A                                 & N/A                                 & N/A                                 \\
             & PENCIL w/o Features (Ours) & {\color{firstbest} 42.23 ± 1.98}    & 47.51 ± 3.09                        & 38.28 ± 2.59                        & {\color{thirdbest} 66.88 ± 0.34}    & {\color{thirdbest} 73.85 ± 0.40}    & 86.74 ± 0.26                        \\
             & PENCIL (Ours)             & 32.12 ± 3.04                        & 43.74 ± 5.47                        & {\color{thirdbest} 38.34 ± 5.14}    & 66.56 ± 0.19                        & {\color{firstbest} 79.54 ± 0.07}    & 86.86 ± 0.20                        \\ \bottomrule
    \end{tabular}%
    }
\end{table*}

\begin{table*}[t]
    \caption{Results under HeaRT. N/A means results are not available. Colored results indicate the \textcolor{firstbest}{\(1^{\text{st}}\)}, \textcolor{secondbest}{\(2^{\text{nd}}\)}, and \textcolor{thirdbest}{\(3^{\text{rd}}\)}-best performance in MRR. Category markers: 
    \raisebox{0.75ex}{\catdot{catgreen}} Pure Heuristic, 
    \raisebox{0.75ex}{\catdot{catblue}} Heuristic-informed, \raisebox{0.75ex}{\catdot{catorange}} ID-based, and \, \raisebox{0.75ex}{\catdot{catred}} Heuristic-agnostic + ID-free.}
    \label{tab:complete-heart}
    \resizebox{\textwidth}{!}{%
    \begin{tabular}{@{}clccccccc@{}}
    \toprule
    Type & Model              & \texttt{cora}                                & \texttt{citeseer}                            & \texttt{pubmed}                             & \texttt{ogbl-collab}                        & \texttt{ogbl-ppa}                            & \texttt{ogbl-ddi}                            & \texttt{ogbl-citation2}                      \\ \midrule
    \multirow{3}{*}{\raisebox{0.75ex}{\catdot{catgreen}}} 
             & CN                  & 9.78 ± 0.00                         & 8.42 ± 0.00                         & 2.28 ± 0.00                        & 4.20 ± 0.00                        & 25.70 ± 0.00                        & 6.71 ± 0.00                         & 17.11 ± 0.00                        \\
             & AA                  & 11.91 ± 0.00                        & 10.82 ± 0.00                        & 2.63 ± 0.00                        & 5.07 ± 0.00                        & 26.85 ± 0.00                        & 6.97 ± 0.00                         & 17.83 ± 0.00                        \\
             & RA                  & 11.81 ± 0.00                        & 10.84 ± 0.00                        & 2.47 ± 0.00                        & {6.29 ± 0.00} & 28.34 ± 0.00                        & 8.70 ± 0.00                         & 17.79 ± 0.00                        \\ \midrule
    \multirow{6}{*}{\raisebox{0.75ex}{\catdot{catblue}}} 
             & SEAL                & 10.67 ± 3.46                        & 13.16 ± 1.66                        & 5.88 ± 0.53                        & {6.43 ± 0.32} & 29.71 ± 0.71                        & 9.99 ± 0.90                         & 20.60 ± 1.28                        \\
             & Neo-GNN             & 13.95 ± 0.39                        & 17.34 ± 0.84                        & 7.74 ± 0.30                        & 5.23 ± 0.90                        & 21.68 ± 1.14                        & 10.86 ± 2.16                        & 16.12 ± 0.25                        \\
             & BUDDY               & 13.71 ± 0.59                        & 22.84 ± 0.36                        & 7.56 ± 0.18                        & 5.67 ± 0.36                        & 27.70 ± 0.33                        & 12.43 ± 0.50                        & 19.17 ± 0.20                        \\
             & NCN                 & 14.66 ± 0.95                        & {\color[HTML]{EA4335} 28.65 ± 1.21} & 5.84 ± 0.22                        & 5.09 ± 0.38                        & 35.06 ± 0.26                        & 12.86 ± 0.78                        & 23.35 ± 0.28                        \\
             & NCNC                & {\color[HTML]{4285F4} 14.98 ± 1.00} & {\color[HTML]{4285F4} 24.10 ± 0.65} & {8.58 ± 0.59} & 4.73 ± 0.86                        & 33.52 ± 0.26                        & N/A                                 & 19.61 ± 0.54                        \\
             & LPFormer            & {\color[HTML]{EA4335} 16.80 ± 0.52} & {\color[HTML]{34A853} 26.34 ± 0.67} & {\color[HTML]{EA4335} 9.99 ± 0.52} & {\color{thirdbest} 7.62 ± 0.26} & 40.25 ± 0.24 & {\color[HTML]{4285F4} 13.20 ± 0.54} & {\color[HTML]{EA4335} 24.70 ± 0.55} \\ \midrule
    \multirow{1}{*}{\raisebox{0.75ex}{\catdot{catorange}}} 
             & MPLP+               & N/A                                 & N/A                                 & N/A                                & {6.79}                               & \color{thirdbest}{41.40}                               & N/A                                 & 23.11                               \\ \midrule
    \multirow{5}{*}{\raisebox{0.75ex}{\catdot{catred}}} 
             & GCN                 & {\color[HTML]{34A853} 16.61 ± 0.30} & 21.09 ± 0.88                        & 7.13 ± 0.27                        & 6.09 ± 0.38                        & 26.94 ± 0.48                        & {\color[HTML]{34A853} 13.46 ± 0.34} & 19.98 ± 0.35                        \\
             & SAGE                & 14.74 ± 0.69                        & 21.09 ± 1.15                        & {\color[HTML]{34A853} 9.40 ± 0.70} & 5.53 ± 0.50                        & 27.27 ± 0.30                        & 12.60 ± 0.72                        & 22.05 ± 0.12                        \\
             & NBFNet              & 13.56 ± 0.58                        & 14.29 ± 0.80                        & N/A                                & N/A                                & N/A                                 & N/A                                 & N/A                                 \\
             & PENCIL w/o Features (Ours) & 14.63 ± 0.52                        & 16.50 ± 0.31                        & 7.05 ± 0.17                        & {\color{secondbest} 8.13 ± 0.09}                        & {\color[HTML]{34A853} 44.57 ± 0.15} & {\color[HTML]{EA4335} 14.07 ± 0.24} & {\color[HTML]{4285F4} 23.36 ± 0.07} \\
             & PENCIL (Ours)              & 13.13 ± 0.53                        & 16.80 ± 1.43                        & {\color{thirdbest}8.88 ± 0.49}                        & {\color{firstbest} 8.21 ± 0.14}                         & {\color[HTML]{EA4335} 45.43 ± 0.31} & N/A                                 & {\color[HTML]{34A853} 23.43 ± 0.14} \\ \bottomrule
    \end{tabular}%
    }
\end{table*}

\subsection{Real-world Experiments} \label{app:real-world-experiments}

We provide additional training details for the results presented in Table~\ref{tab:original-setting} and Table~\ref{tab:heart}. For these experiments, we employed a minimal hyperparameter tuning strategy, focusing primarily on model depth and sampling configurations to ensure adequate capture of local neighborhood information. Additionally, we freeze the input projection matrix $\mathbf{W}_0$ across all runs, as we observed no significant impact on performance when it was learnable. Batch sizes were scaled according to hardware memory constraints. As detailed in Table~\ref{tab:original_configs} and Table~\ref{tab:heart-config}, configurations are largely consistent across datasets; the notable exception is \texttt{citeseer}, which utilizes a larger hidden dimension and a reduced depth of two layers. For a fair comparison with the original study~\cite{li_evaluating_2023}, we use a single negative link for each positive link for all datasets. For \texttt{ogbl-ppa}, the HeaRT evaluation protocol \cite{li_evaluating_2023} is computationally prohibitive due to the customized negative links per positive edge, as discussed in Section~\ref{sec:experimental-results}. Consequently, only for this dataset, we utilized the optimal checkpoint identified in the original benchmark setting for the final HeaRT evaluation, leveraging the fact that training splits are identical across settings. All experiments were parallelized across 4/8 GPUs, utilizing NVIDIA A5000/A6000 cards for Planetoid benchmarks \cite{yang_revisiting_2016} and A100 cards for OGB benchmarks \cite{hu2020ogb}. Table~\ref{tab:complete-original-setting} and Table~\ref{tab:complete-heart} supplement the experimental results provided in the main text, where CN, AA, and RA are added for completeness.

\subsection{Ablation Study on Multiplicative Residual Connection} \label{app:mr-ablation-study}

\begin{table*}[t]
    \caption{Ablation study on Multiplicative Residual (MR).}
    \label{tab:multplicative-residual}
    \resizebox{\textwidth}{!}{%
    \begin{tabular}{@{}cccccccccc@{}}
    \toprule
    \multirow{2}{*}{Model}   & \multicolumn{3}{c}{\texttt{cora (2, 20)}}           & \multicolumn{3}{c}{\texttt{pubmed (2, 16)}}         & \multicolumn{3}{c}{\texttt{ogbl-collab (1, 75)}}    \\ \cmidrule(l){2-10} 
                             & MRR          & H@3          & H@20         & MRR          & H@3          & H@20         & H@20         & H@50         & H@100        \\ \midrule
    PENCIL \textbackslash MR & 34.64 ± 1.99 & 35.98 ± 3.97 & 58.67 ± 2.57 & 21.49 ± 4.30 & 21.81 ± 2.81 & 31.81 ± 2.32 & 49.52 ± 3.01 & 53.43 ± 1.92 & 56.04 ± 1.24 \\
    PENCIL                   & 42.23 ± 1.98 & 43.34 ± 1.03 & 67.32 ± 3.56 & 38.28 ± 2.59 & 43.61 ± 1.48 & 57.49 ± 0.94 & 53.47 ± 0.85 & 66.88 ± 0.34 & 69.75 ± 0.45 \\ \midrule
    Performance Gain         & +7.59         & +7.36         & +8.65         & +16.79        & +21.80        & +25.68        & +3.95         & +13.45        & +13.71        \\ \bottomrule
    \end{tabular}%
    }
\end{table*}

We perform an ablation study to isolate the contribution of the \textbf{multiplicative residual} defined in Section~\ref{sec:plain-transformers-as-link-predictors}. Contemporary approaches often attempt to inject graph inductive biases into plain Transformers solely through PEs/SEs \cite{dwivedi2021generalization, sanford_understanding_2024, yehudai_depth-width_2025, ma_plain_2025}. However, Table~\ref{tab:multplicative-residual} reveals that relying on input encodings alone is sub-optimal; the inclusion of the multiplicative residual provides substantial performance improvements across all datasets, indicating that an explicit structural prior is necessary for every layer.

\subsection{Effects of Input Projection Matrix Initialization} \label{app:effects-of-input-projection-matrix-initialization}

\begin{figure}[t]
    \centering
    \includegraphics[width=\linewidth]{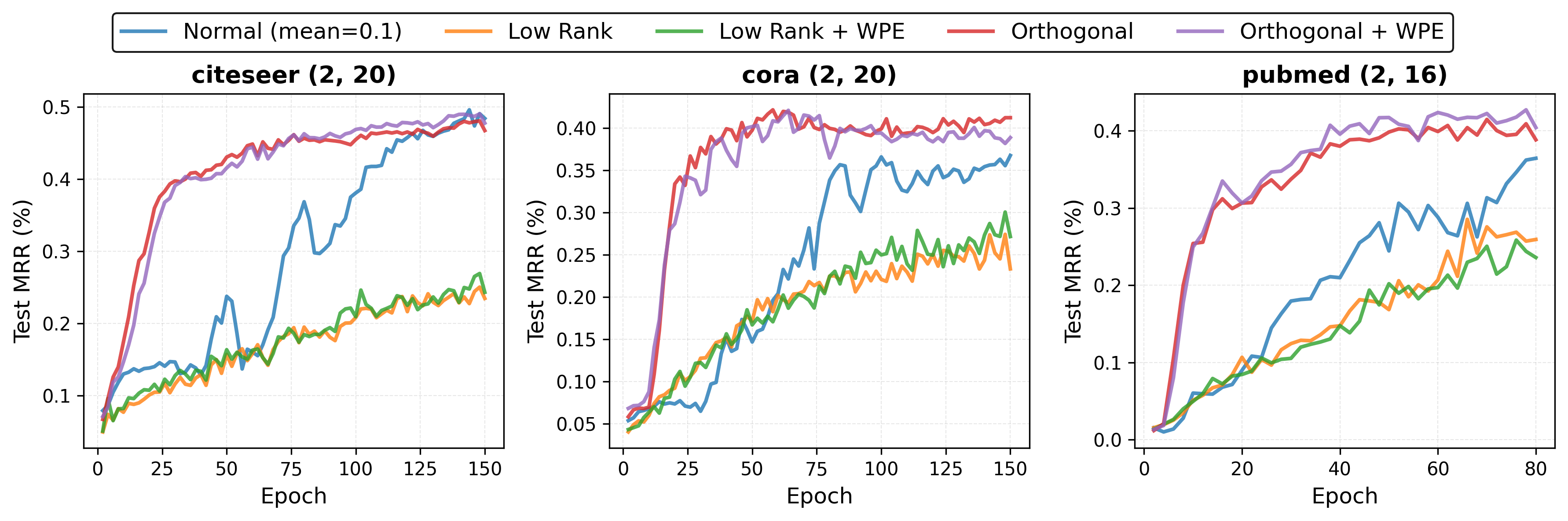}
    \caption{MRR averaged over 5 runs on the \texttt{citeseer}, \texttt{cora}, and \texttt{pubmed} datasets using different input embeddings' initializations.}
    \label{fig:averaged_test_mrr}
\end{figure}

In this section, we study how the initialization of the input projection matrix $\mathbf{W}_0$ influences optimization and performance. We compare (i) \textbf{orthogonal} initialization (our default), (ii) a \textbf{mean-shifted Gaussian} initialization with mean $0.1$ to probe sensitivity to non-zero-mean embeddings, and (iii) a \textbf{low-rank} initialization with rank $5$, which constrains the embedding space and limits the number of mutually orthogonal directions. Following the same implementation as discussed earlier, we freeze $\mathbf{W}_0$. This also isolates the initialization scheme as the sole variable affecting optimization. All experiments use the same hyperparameters as Table~\ref{tab:original_configs}. Figure~\ref{fig:averaged_test_mrr} reports test MRR averaged over 5 runs on \texttt{citeseer}, \texttt{cora}, and \texttt{pubmed}. Orthogonal initialization yields the fastest and most stable convergence and achieves the strongest final MRR across datasets, whereas low-rank initialization substantially degrades performance and converges to markedly lower MRR. The mean-shifted Gaussian initialization exhibits noticeably higher variance and slower, less stable training. We also evaluate adding learnable word position embeddings (WPE). Across all settings, WPE induces at most minor and dataset-dependent differences and does not consistently improve either convergence or final MRR, suggesting that---under this graph input encoding---sequential positional inductive biases are irrelevant for performance.

\subsection{Effect of Multiple Labeled Subgraphs}
\label{app:multiple_labeled_subgraphs}

We further examine whether using multiple labeled subgraphs per candidate link improves performance. Increasing the number of such subgraphs can substantially increase memory cost, since each additional subgraph requires an additional forward pass. In Table~\ref{tab:multiple_labeled_subgraphs}, we compare PENCIL with a variant that applies DeepSets-style pooling~\cite{zaheer_deep_2017} over three labeled subgraphs per sample following the graph encoding scheme in Section~\ref{sec:graph-encoding-scheme}. We do not observe clear performance gains from using multiple labeled subgraphs. Moreover, as noted by Zhou et al.~\cite{zhou_relational_2023}, selecting multiple labeled subgraphs often requires additional node-selection policies, introducing further computational overhead.

\begin{table}[h]
    \centering
    \caption{Effect of using multiple labeled subgraphs. We compare PENCIL with a variant that applies DeepSets-style pooling over three labeled subgraphs per sample. Higher values are better.}
    \label{tab:multiple_labeled_subgraphs}
    \begin{tabular}{lccc}
        \toprule
        \multirow{2}{*}{Model}
        & \texttt{cora (2, 20)}
        & \texttt{citeseer (2, 20)}
        & \texttt{ogbl-collab (1, 75)} \\
        \cmidrule(lr){2-4}
        & MRR & MRR & H@50 \\
        \midrule
        PENCIL & \(42.23 \pm 1.98\) & \(47.51 \pm 3.09\) & \(66.88 \pm 0.34\) \\
        PENCIL + LRP-3 & 39.18 & 42.17 & 67.05 \\
        \bottomrule
    \end{tabular}
\end{table}

\subsection{Link Prediction on Latent Space Graphs} \label{app:latent-space-graphs}

In this section, we further evaluate PENCIL and several simple baselines on synthetic latent-space graphs~\cite{hoff2002latent}. This controlled setting allows us to separately vary the strength of structural signal and node-feature signal for link prediction.

We generate a latent-space graph with \(N\) nodes and \(K\) clusters as follows. Each node \(i\) is first assigned a cluster label \(y_i\). Conditional on this label, its latent position is sampled from a Gaussian distribution:
\[
\mathbf{z}_i \mid y_i \sim \mathcal{N}(\mathbf{c}_{y_i}, \sigma_{\mathrm{noise}}^2 I),
\]
where the cluster centers \(\mathbf{c}_k \in \mathbb{R}^{d_{\mathrm{latent}}}\) are placed along coordinate axes with magnitude \(s\). Given the latent positions, each undirected edge is independently sampled according to
\[
A_{ij} \sim \mathrm{Bernoulli}
\left(
\sigma\left(\alpha - \beta \|\mathbf{z}_i-\mathbf{z}_j\|_2^2\right)
\right),
\]
where \(\alpha\) controls the overall graph density and \(\beta\) controls the sensitivity of edge formation to latent distance.

Node features are generated as noisy random projections of the latent positions. Let
\(\mathbf{W} \in \mathbb{R}^{d_{\mathrm{latent}} \times p}\) be a random projection matrix with entries \(W_{ab} \sim \mathcal{N}(0,1/p)\). The feature vector of node \(i\) is then given by
\[
\mathbf{x}_i
=
\sqrt{\mu}\,\mathbf{z}_i \mathbf{W}
+
\boldsymbol{\eta}_i,
\qquad
\boldsymbol{\eta}_i \sim \mathcal{N}(\mathbf{0}, I_p/p).
\]
Here, \(\mu\) controls the signal-to-noise ratio of the node features: larger values of \(\mu\) make the observed features more informative about the latent position.

Under this generative process, the cluster label \(y_i\) determines the distribution of the latent position \(\mathbf{z}_i\), and \(\mathbf{z}_i\) subsequently determines both the edge probabilities and the node features. Thus, structural information and node-feature information are both induced by the same latent geometry. The parameter \(s\) controls the separation among cluster centers, thereby modulating the strength of latent structural homophily, while \(\mu\) controls how strongly the latent positions are reflected in the observed node features.

In our experiments, we fix \(N = 5000\), \(K = d_{\mathrm{latent}} = 3\), \(p = 32\), \(\beta = 2\), and \(\sigma_{\mathrm{noise}} = 0.5\), and vary only \(s\) and \(\mu\). For each configuration, we dynamically adjust \(\alpha\) so that the generated graph has average degree approximately \(15\). We use \(90\%\) of observed positive edges for training and hold out the remaining \(10\%\) for testing. For evaluation, we construct a balanced link-prediction dataset by sampling an equal number of positive and negative node pairs. Negative pairs are selected from non-edges whose endpoints are far apart in latent space, yielding a controlled setting in which positive and negative examples are well separated by latent distance. Figures~\ref{fig:latent_position} and~\ref{fig:latent_all_edges} visualize the latent geometry and distance distributions under our experimental settings.

\begin{figure}[t]
    \centering
    \includegraphics[width=\linewidth]{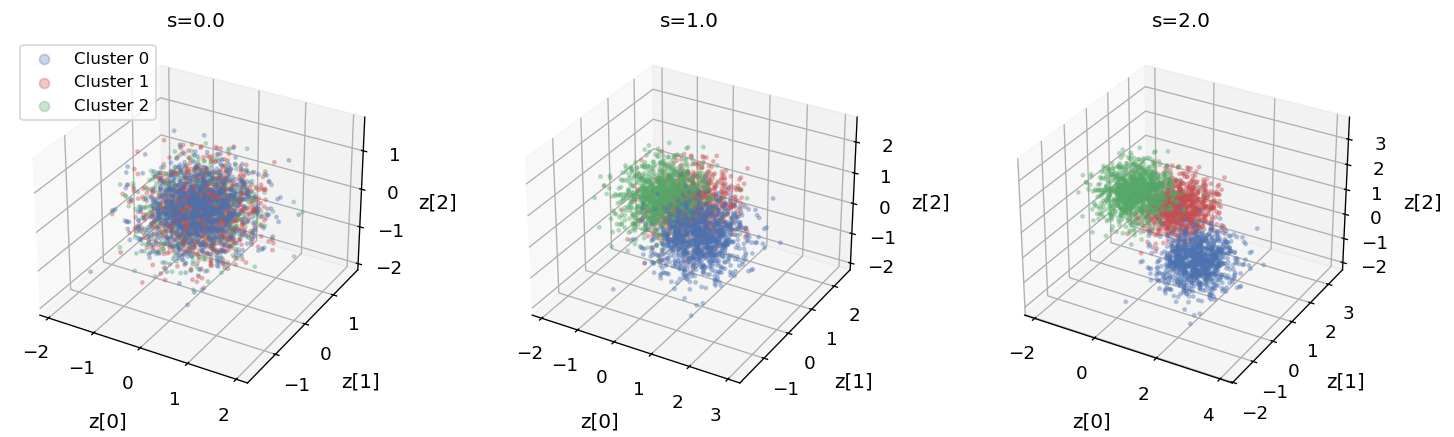}
    \caption{Latent geometry under varying cluster separation \(s\).}
    \label{fig:latent_position}
\end{figure}

\begin{figure}[t]
    \centering
    \includegraphics[width=\linewidth]{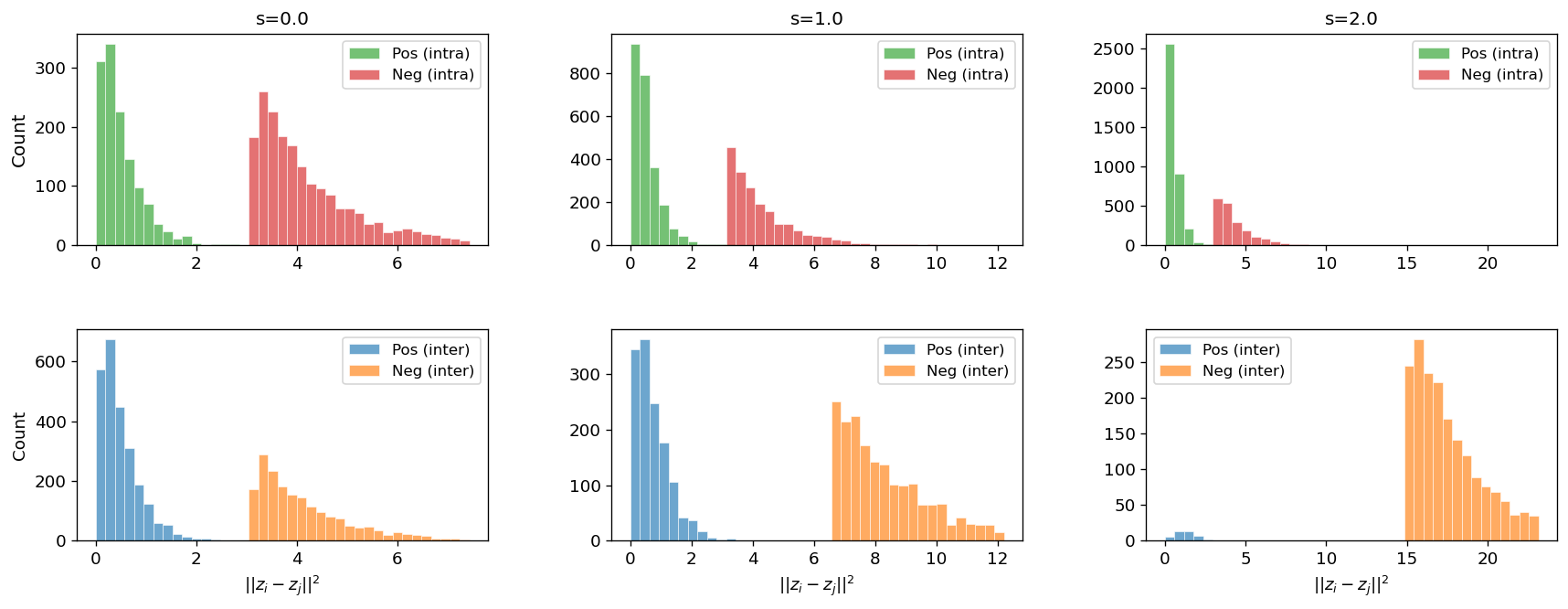}
    \caption{Squared latent-distance distributions for positive and negative intra-/inter-cluster pairs under varying \(s\).}
    \label{fig:latent_all_edges}
\end{figure}

We evaluate two experimental axes: varying the structural signal \(s\) while fixing \(\mu = 1\), and varying the node-feature signal \(\mu\) while fixing \(s = 1\). We additionally consider two edge cases, \((s = 2,\mu = 0)\) and \((s = 0,\mu = 8)\), corresponding to strong structural signal without informative node features and strong node-feature signal without cluster separation, respectively. Larger values of \(s\) and \(\mu\) indicate stronger structural and node-feature signals. All neural models use two layers with hidden dimension fixed to \(256\). The experiments are reported in Table~\ref{tab:latent_vary_s} and~\ref{tab:latent_vary_mu}.

\begin{table}[t]
    \centering
    \caption{Link-prediction performance on latent-space graphs when varying the structural signal \(s\) while fixing \(\mu=1\), with the additional edge case \((s,\mu)=(2,0)\). Reported metric is MRR.}
    \label{tab:latent_vary_s}
    \begin{tabular}{lcccc}
        \toprule
        Model & \((s=0, \mu=1)\) & \((s=1, \mu=1)\) & \((s=2, \mu=1)\) & \((s=2, \mu=0)\) \\
        \midrule
        CN     & 1.58  & 4.78  & 4.48  & 4.48 \\
        AA     & 2.85  & 5.43  & 7.37  & 7.37 \\
        GCN    & 75.48 & 70.08 & 83.68 & \textbf{40.67} \\
        GAT    & 71.44 & \textbf{81.29} & 82.11 & 6.27 \\
        PENCIL & \textbf{77.63} & 74.37 & \textbf{93.59} & 36.01 \\
        \bottomrule
    \end{tabular}
\end{table}

\begin{table}[t]
    \centering
    \caption{Link-prediction performance on latent-space graphs when varying the node-feature signal \(\mu\) while fixing \(s=1\), with the additional edge case \((s,\mu)=(0,8)\). Reported metric is MRR.}
    \label{tab:latent_vary_mu}
    \begin{tabular}{lcccc}
        \toprule
        Model & \((s=1, \mu=0)\) & \((s=1, \mu=2)\) & \((s=1, \mu=8)\) & \((s=0, \mu=8)\) \\
        \midrule
        CN     & 4.78  & 4.78  & 4.78  & 1.58 \\
        AA     & 5.43  & 5.43  & 5.43  & 2.85 \\
        GCN    & 10.09 & 71.04 & 73.85 & 77.05 \\
        GAT    & 5.36  & 83.57 & 83.04 & 76.01 \\
        PENCIL & \textbf{17.21} & \textbf{84.90} & \textbf{92.05} & \textbf{95.34} \\
        \bottomrule
    \end{tabular}
\end{table}

We make several observations. First, pairwise structural heuristics such as CN and AA perform poorly across all settings, indicating that local overlap-based scores alone are insufficient to recover the latent-distance signal. Second, neural models benefit substantially from informative node features, with performance generally improving as \(\mu\) increases. Third, although GAT and PENCIL are both attention-based models, their attention mechanisms use fundamentally different inputs. GAT computes attention weights for feature aggregation over local neighborhoods, so its attention scores are directly tied to node-feature quality; consequently, its performance drops sharply when the features are uninformative \((\mu=0)\). In contrast, PENCIL applies self-attention over a tokenized link-centric subgraph whose tokens explicitly encode the graph structure. This allows PENCIL to attend over pairwise relationships among nodes in the sampled subgraph, rather than relying only on feature-based neighbor weighting. Finally, PENCIL achieves the best performance in most settings, suggesting that it can effectively leverage node-feature cues when they are informative while retaining access to structural cues through its subgraph tokenization.

\subsection{Computational Complexity} \label{app:computational-complexity}

In this section, we analyze the computational complexity of PENCIL and GAT under three perspectives: batching, forward pass, and node labeling overhead. Our goal is to show that PENCIL remains computationally efficient across these dimensions, owing to its NLP-style batching scheme, efficient attention kernels, and labeling-free architecture.

\subsubsection{Batching Overhead}

We benchmark the overhead of constructing mini-batches of \emph{isolated} sampled subgraphs and the associated peak RAM usage, comparing a padding-and-stacking implementation based on native tensor operations against PyG collation routines \cite{fey2025pyg,fey2019pyg}. This setting differs from PyG's neighbor-sampling loaders (e.g., \texttt{NeighborLoader} and \texttt{LinkNeighborLoader}), which typically construct a single computation graph for a batch of seed nodes/edges and thus \emph{merge} overlapping neighborhoods into one (potentially larger) subgraph. In contrast, PENCIL operates on a fixed-budget collection of independent subgraph instances; even if two samples share underlying nodes in the original graph, they are treated as separate instances after sampling. Since this isolated subgraph list representation is not the typical output of PyG's neighbor-sampling loaders, we use PyG's \texttt{from\_data\_list} as a GNN-style batching baseline to collate a list of variable-sized \texttt{Data} objects extracted by the ShaDowKHop sampler \cite{zeng2021decoupling}, and compare it against a Transformer-style batching procedure that pads each subgraph to a fixed budget and stacks the resulting tensors. Figure~\ref{fig:batching_time_and_memory_comparison} reports batching time and peak RAM on \texttt{ogbl-collab} as a function of batch size. The padding-and-stacking approach yields substantially lower batching overhead and consistently low peak RAM, whereas \texttt{from\_data\_list} incurs significantly higher batching time and rapidly increasing memory consumption as batch size grows. We attribute this gap primarily to implementation-level batching overhead: padding-and-stacking produces contiguous dense tensors with a fixed shape, while \texttt{from\_data\_list} must concatenate many variable-sized graph objects (including index offsetting and allocation of batched sparse structures), leading to higher collation cost and memory footprints that scale with the total number of nodes/edges in the batch.

\subsubsection{Training/Inference Time}

We benchmark per-mini-batch training and inference time for PENCIL and GAT on \texttt{pubmed} and \texttt{ogbl-collab} under identical sampling and model configurations. Using the same sampling configuration isolates how runtime depends on graph sparsity: \texttt{ogbl-collab} exhibits higher connectivity than \texttt{pubmed} (e.g., higher average degree; see Table~7 in \cite{li_evaluating_2023}). Table~\ref{tab:training-inference-time} reports wall-clock time (seconds; mean $\pm$ std) when varying depth. On \texttt{pubmed}, GAT is consistently faster than PENCIL for both training and inference. On \texttt{ogbl-collab}, however, the two models have comparable per-batch costs, with the gap substantially reduced. This trend is consistent with the underlying computation: GAT computes attention only on sampled edges, so its cost grows with the number of edges present in sampled subgraphs, whereas PENCIL computes dense attention across all tokens with each sampled subgraph instance. Notably, despite having roughly $7$--$8\times$ more parameters than GAT in these settings and using dense self-attention, PENCIL remains competitive, benefiting from highly optimized Transformer implementations \cite{wolf-etal-2020-transformers,dao2022flashattention}. Finally, linear-attention variants could further reduce the compute and memory footprint of self-attention during both training and inference for large token budgets \cite{choromanski2021rethinking}; we leave this extension to future work.

\begin{figure}[t]
    \centering
    \includegraphics[width=0.9\linewidth]{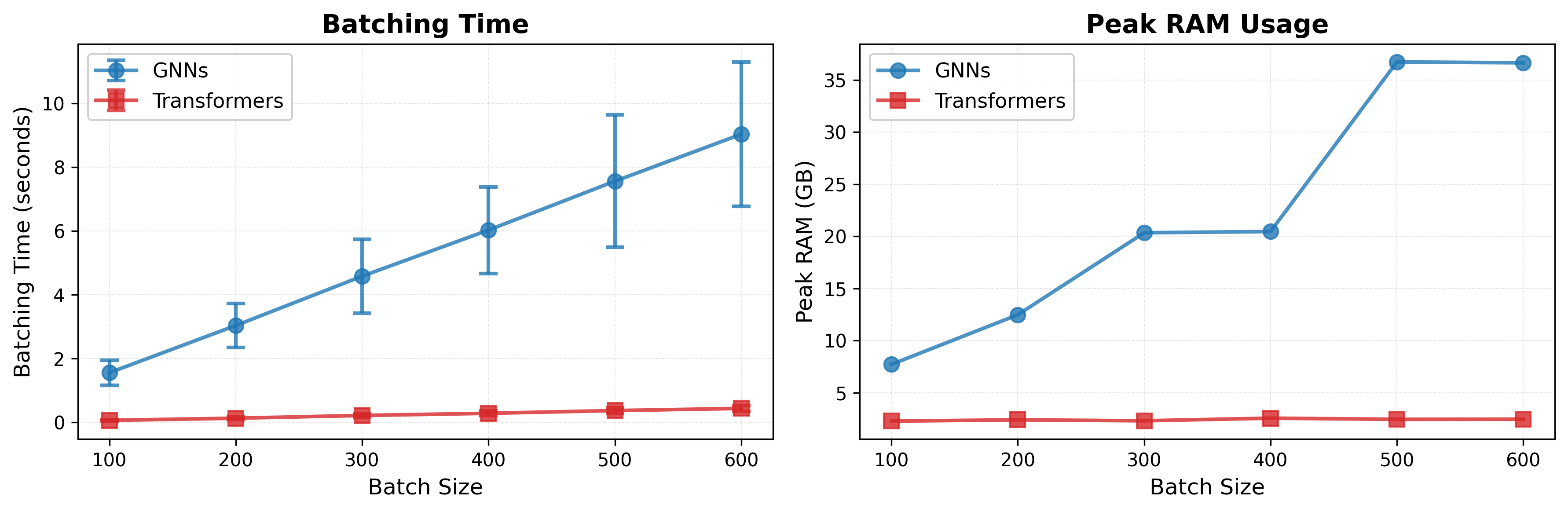}
    \caption{Batching time and memory comparison of Transformers and GNNs on the \texttt{ogbl-collab} dataset.}
    \label{fig:batching_time_and_memory_comparison}
\end{figure}

\begin{table}[tbp]
    \caption{Number of parameters and wall-clock time per mini-batch (seconds; mean $\pm$ std) for training and inference of PENCIL and GAT with 3 and 8 layers on \texttt{ogbl-collab} and \texttt{pubmed} under the $(1, 75)$ sampling configuration. All measurements are collected on the same hardware configuration, using batch size 100, and averaged over 50 batches.}

    \label{tab:training-inference-time}
    \resizebox{\textwidth}{!}{%
    \begin{tabular}{@{}cccccc@{}}
    \toprule
    \multirow{2}{*}{Model} & \multirow{2}{*}{\# Params} & \multicolumn{2}{c}{\texttt{ogbl-collab (1, 75)}}            & \multicolumn{2}{c}{\texttt{pubmed (1, 75)}}                 \\ \cmidrule(l){3-6} 
                           &                            & Training Time Per Batch & Inference Time Per Batch & Training Time Per Batch & Inference Time Per Batch \\ \midrule
    PENCIL-3L                & 10.2M                      & 0.019 ± 0.045           & 0.012 ± 0.039            & 0.016 ± 0.036           & 0.010 ± 0.029            \\
    PENCIL-8L                & 27.3M                      & 0.035 ± 0.049           & 0.020 ± 0.043            & 0.029 ± 0.038           & 0.016 ± 0.031            \\
    GAT-3L                 & 1.32M                      & 0.021 ± 0.078           & 0.015 ± 0.070            & 0.011 ± 0.033           & 0.007 ± 0.026            \\
    GAT-8L                 & 3.95M                      & 0.032 ± 0.079           & 0.017 ± 0.047            & 0.020 ± 0.034           & 0.011 ± 0.027            \\ \bottomrule
    \end{tabular}%
    }
\end{table}

\subsubsection{Node Labeling Overhead}
\label{app:subgraph_extraction_overhead}

We further measure the computational overhead of subgraph extraction under commonly used structural and positional encodings. Specifically, we evaluate extraction time and memory usage on a synthetic graph with \(10{,}000\) nodes and average degree \(15\), comparing unlabeled subgraphs with DRNL, Random Walk Positional Encoding (RWPE) of length \(8\), and Laplacian Positional Encoding (LapPE) with \(k=4\), which are commonly used in SEAL-style and graph Transformer methods. As shown in Table~\ref{tab:subgraph_extraction_overhead}, PEs introduce substantial overhead, and even DRNL requires approximately \(1.6\times\) to \(4\times\) more extraction time than the unlabeled subgraphs used by PENCIL.

\begin{table}[h]
    \centering
    \caption{Subgraph extraction overhead on a \(10{,}000\)-node graph with average degree \(15\). Time is reported in milliseconds and memory is reported in megabytes.}
    \label{tab:subgraph_extraction_overhead}
    \begin{tabular}{lcccccccccc}
        \toprule
        \multirow{2}{*}{Sampling config}
        & \multirow{2}{*}{Maximum \# nodes}
        & \multicolumn{2}{c}{No labeling}
        & \multicolumn{2}{c}{DRNL}
        & \multicolumn{2}{c}{RWPE}
        & \multicolumn{2}{c}{LapPE} \\
        \cmidrule(lr){3-4}
        \cmidrule(lr){5-6}
        \cmidrule(lr){7-8}
        \cmidrule(lr){9-10}
        & & Time & Mem. & Time & Mem. & Time & Mem. & Time & Mem. \\
        \midrule
        \((1,20)\) & 42   & 0.109 & 0.07 & 0.430 & 0.19 & 0.161 & 0.05 & 0.270 & 0.07 \\
        \((1,30)\) & 62   & 0.132 & 0.05 & 0.454 & 0.05 & 0.184 & 0.05 & 0.387 & 0.11 \\
        \((2,20)\) & 842  & 1.392 & 0.07 & 2.225 & 0.35 & 6.980 & 0.07 & 6.889 & 0.24 \\
        \((2,30)\) & 1862 & 4.155 & 0.17 & 5.676 & 1.10 & 60.371 & 0.17 & 20.846 & 0.50 \\
        \bottomrule
    \end{tabular}
\end{table}

\end{document}